\documentclass{article}

\usepackage{algorithm}  
\usepackage{algpseudocode}
\usepackage{amsfonts}
\usepackage{amsmath}
\usepackage{amsthm}
\usepackage{bbm}
\usepackage{geometry}
\usepackage{graphicx}
\usepackage[colorlinks=true]{hyperref} 
\usepackage{multirow}
\usepackage{subcaption}
\usepackage{tabu}
\usepackage[dvipsnames]{xcolor}
\usepackage{titling}

\posttitle{\par\end{center}\vspace{-4em}}

\hypersetup{urlcolor=blue, citecolor=red}

\newtheorem{corollary}{Corollary}

\newtheorem{proposition}{Proposition}

\title{Geometric adaptive Monte Carlo in random environment}

\date{}

\begin{document}
\maketitle

\centerline{\scshape Theodore Papamarkou$^{1,2}$,
Alexey Lindo$^{2}$,
Eric B. Ford$^{3,4,5,6}$
}

\medskip

{\footnotesize
\centerline{$^{1}$Department of Mathematics}
\centerline{The University of Manchester, Manchester, M13 9PL, UK}
} 

\medskip

{\footnotesize
\centerline{$^{2}$School of Mathematics and Statistics}
\centerline{University of Glasgow, Glasgow, G12 8QQ, UK}
}

\medskip

{\footnotesize
\centerline{$^{3}$Department of Astronomy and Astrophysics}
\centerline{525 Davey Laboratory, The Pennsylvania
State University, University Park, PA, 16802, USA}
}

\medskip

{\footnotesize
\centerline{$^{4}$Center for Exoplanets and Habitable Worlds}
\centerline{525 Davey Laboratory, The Pennsylvania
State University, University Park, PA, 16802, USA}
}

\medskip

{\footnotesize
\centerline{$^{5}$Center for Astrostatistics}
\centerline{525 Davey Laboratory, The Pennsylvania
State University, University Park, PA 16802, USA}
}

\medskip

{\footnotesize
\centerline{$^{6}$Institute for Computational and Data Sciences}
\centerline{525 Davey Laboratory, The Pennsylvania
State University, University Park, PA 16802, USA}
}

\bigskip

%
%
%
%
%


\begin{abstract}
Manifold Markov chain Monte Carlo algorithms have been introduced to sample 
more effectively from challenging target 
densities exhibiting multiple modes or strong correlations. Such algorithms 
exploit the local geometry of the parameter 
space, thus enabling chains to achieve a faster convergence rate when measured 
in number of steps.
However, acquiring 
local geometric information
can often increase computational complexity per step to the extent that 
sampling from high-dimensional 
targets becomes inefficient in terms of total computational time.
This paper analyzes the computational complexity of manifold Langevin Monte 
Carlo and proposes a geometric adaptive Monte
Carlo sampler aimed at balancing the benefits of exploiting local geometry with 
computational cost to achieve a high 
effective sample size for a given computational cost.
The suggested sampler is a discrete-time stochastic process in random 
environment.
The random environment allows to switch between local geometric and adaptive 
proposal kernels with the help of a schedule.
An exponential schedule is put forward that enables more frequent use of 
geometric information in early 
transient phases of the chain, while saving computational time in late 
stationary phases. 
The average complexity can be manually set depending on the need for geometric 
exploitation posed by the underlying model.
\end{abstract}

\section{Introduction}

Geometric Markov chain Monte Carlo (MCMC) dates back to the work of 
\cite{dua_ken_pen__hyb},
which introduced Hamiltonian Monte Carlo (HMC) to unite MCMC with molecular 
dynamics.
Statistical applications of HMC began with its use in neural network models by 
\cite{nea__bay}.

In the meanwhile, the Metropolis-adjusted Langevin algorithm (MALA) was 
proposed by \cite{rob_ros__opt} to employ Langevin 
dynamics for MCMC sampling. Both HMC and MALA evaluate the gradient of the 
target density, so 
they utilize local geometric flow.

\cite{gir_cal__rie} introduced new differential geometric MCMC methods. Given a
state $\theta\in\mathbb{R}^{n}$, \cite{gir_cal__rie} defines a distance between 
two probability 
densities $p(\theta)$ and $p(\theta+\delta\theta)$ as the quadratic form
$\delta\theta^T M(\theta) \delta\theta$ for an arbitrary metric 
$M(\theta)$. Thus, the position-specific metric $M(\theta)$ induces a Riemann 
manifold in the 
space of parameterized probability density functions 
$\left\{p(\theta):\theta\right\}$.
\cite{gir_cal__rie} uses $M(\theta)$ to define proposal kernels that explore 
the state space
$\left\{\theta:\theta\in\mathbb{R}^{n}\right\}$ effectively by introducing 
Riemann 
manifold Langevin and Hamiltonian Monte Carlo methods.

Computing the geometric entities of differential geometric MCMC methods creates 
a performance bottleneck that restricts the 
applicability of the involved methods. For example, manifold MCMC methods 
require to calculate the metric tensor
$M(\theta)$ of choice. Typically, $M(\theta)$ is set to be the observed Fisher 
information matrix,
which equals the negative Hessian of the log-target density at state $\theta$. 
Consequently, the 
complexity of manifold MCMC algorithms is dominated by Hessian-related 
computations,
such as the gradient of or the inverse of the Hessian.

\cite{gir_cal__rie} constructed the simplified manifold Metropolis-adjusted 
Langevin algorithm (SMMALA) that is of the same 
order of complexity per Monte Carlo iteration but faster than MMALA and RMHMC 
for target distributions of low complexity. 
The faster speed of SMMALA over MMALA and RMHMC is explained by lower order 
terms and constant factors appearing in big-oh 
notation, which are ordinarily omitted but affect runtime in the case of less 
costly targets.

SMMALA has 
been used in conjunction with population MCMC for the Bayesian analysis of 
mechanistic models based on systems of 
non-linear differential equations, see \cite{cal_gir__sta,sch_pap__ews}. 
Despite the capacity of SMMALA to 
exploit local geometric information so as to cope with non-linear correlations 
and modest increase in the complexity of the 
target density, in case of more expensive targets its computational complexity 
can render performance inferior to other 
algorithms such as the Metropolis-adjusted Langevin algorithm (MALA) or 
adaptive MCMC, see \cite{cal_eps_sil__bay}.

Various attempts have been made to ameliorate the computational implications of 
geometric MCMC methods. Along these lines, 
\cite{lan_tha_chr__emu} used Gaussian processes to emulate the Hessian matrix 
and Christoffel symbols associated with 
the observed Fisher information $M(\theta)$. \cite{sim_bad_cem__sto} developed 
a stochastic quasi-Newton 
Langevin Monte Carlo algorithm which takes into account the local geometry, 
while approximating the inverse Hessian by using 
a limited history of samples and their gradients. Alternatively, 
\cite{per__prox} used convex analysis and proximal 
techniques instead of differential calculus in order to construct a Langevin 
Monte Carlo method for high-dimensional target 
distributions that are log-concave and possibly not continuously differentiable.

The present paper serves two purposes.
Initially, it studies the computational complexity of geometric Langevin Monte 
Carlo algorithms.
Subsequently, it develops the so-called geometric adaptive Monte Carlo (GAMC) 
sampling scheme based on a random environment of geometric and adaptive 
proposal kernels.
For expensive targets, a carefully selected random environment of a local 
geometric Langevin kernel and of adaptive 
Metropolis kernels can give rise to a sampler with higher sampling efficacy 
than a sampler based on the local geometric 
Langevin kernel alone.

\section{Background}
\label{background}

The role of this section is to provide a brief overview of Langevin Monte Carlo 
and of adaptive Metropolis, which will
be combined in later sections to construct a Monte Carlo sampling method in 
random environment.

To establish notation,
consider a Polish space $E$ equipped with the Borel $\sigma$-algebra
$\mathcal{E}$.
Let $\{\theta_k\}$ be a sequence in $E$
and
\begin{equation*}
\theta_{w:z}:=\left(\theta_{w},\theta_{w+1},\ldots,\theta_{z}\right),~w\le z,
\end{equation*}
a subsequence of $\{\theta_k\}$, where $k$, $w$ and $z$ denote non-negative 
integers.
Moreover, let
$Q\colon E^{z-w+1}\times\mathcal{E}\rightarrow\mathbb{R}_{+}$
be a kernel
from $E^{z-w+1}$ to $E$.

It is assumed that for every subsequence $\theta_{w:z}$
of $\{\theta_k\}$, measure $Q(\theta_{w:z},\cdot)$ is absolutely continuous
with respect to some measure $\nu$
on $(E^{z-w+1},\mathcal{E}^{z-w+1})$.
Such an assumption ensures existence of the Radon-Nikodym derivative 
$q_{\theta_{w:z}}$ so that
\begin{equation}
\label{Q}
Q(\theta_{w:z}, B)=
\int_{B}q_{\theta_{w:z}}d\nu
\end{equation}
for any $B\in\mathcal{E}$.

The Radon-Nikodym derivative $q_{\theta_{w:z}}$ is used in the paper as a 
proposal density in geometric or adaptive 
Monte Carlo methods to sample from a possibly unnormalized target density 
$p\colon E\rightarrow \mathbb{R}_{+}$.
In the context of Monte Carlo sampling, $\{\theta_k\}$ and $Q$ are called chain 
and proposal kernel, respectively.

\subsection{Basics of Langevin Monte Carlo}

Langevin Monte Carlo (LMC) is a case of Metropolis-Hastings.
The normal proposal density at the $k$-th iteration of LMC is given by
\begin{equation}
\label{langevin_proposal}
g_{\theta_k}(\theta^{\star}) :=
\mathcal{N}(\theta^{\star}|\mu(\theta_k, M, \epsilon), \Sigma(M, \epsilon)),
\end{equation}
where $\theta_k$ and $\theta^{\star}$ denote the state at the $k$-th iteration 
and the proposed state, respectively.
$M$ is a positive definite matrix of size $n \cdot n$ and
$\epsilon$ refers to a tuning parameter known as the
integration stepsize. The location $\mu(\theta_k, M, \epsilon)$ is a function 
of 
$\theta_k$, $M$ and $\epsilon$, whereas the covariance $\Sigma(M, \epsilon)$ of 
the proposal kernel depends on 
$M$ and $\epsilon$. Both $\mu(\theta_k, M, \epsilon)$ and $\Sigma(M, \epsilon)$ 
are defined so that 
the proposed states admit a Langevin diffusion approximated by a first-order 
Euler discretization.
The Metropolis-Hastings acceptance probability is set to its standard form
\begin{equation}
\label{langevin_acceptance}
r_{g}(\theta_k,\theta^{\star}) :=
\min\left\{
\frac{p(\theta^{\star})g_{\theta^{\star}}(\theta_k)}
{p(\theta_k)g_{\theta_k}(\theta^{\star})}
, 1
\right\}
\end{equation}
if $p(\theta_k)g_{\theta_k}(\theta^{\star})>0$, and 
$r_{g}(\theta,\theta^{\star}):=1$ otherwise.

The proposal kernel $G(\theta_k,\cdot)$ corresponding to the normal density 
$g_{\theta_k}$ of equation
\eqref{langevin_proposal} is defined by setting
$E=\mathbb{R}^n$, $w=z=k$, $q_{\theta_{w:z}}=g_{\theta_k}$
and $\nu$ to be the Lebesgue measure in equation \eqref{Q}.

The integration stepsize $\epsilon$, also known as drift step, is associated 
with the first order Euler discretization 
and significantly affects the rate of state space exploration. If $\epsilon$ is 
selected to be relatively large, many 
of the proposed candidates will be far from the current state, and are likely 
to have a low probability of acceptance, 
so the chain with proposal kernel $G$ will have low acceptance rate. Reducing 
$\epsilon$ will increase the acceptance 
rate, but the chain will take longer to traverse the state space.

In a Bayesian setting, the target is a possibly unnormalized posterior density 
$p(\cdot|y)$,
where $y$ denotes the available data. Replacing $p(\cdot)$ by $p(\cdot|y)$ in
\eqref{langevin_acceptance} makes Langevin Monte Carlo applicable in Bayesian 
problems.

To fully specify a Langevin Monte Carlo algorithm, the location $\mu(\theta_k, 
M, \epsilon)$ and 
covariance $\Sigma(M, \epsilon)$ of normal proposal \eqref{langevin_proposal} 
need to be defined. In what follows, 
variations of geometric Langevin Monte Carlo methods are distinguished by their 
respective proposal location and covariance.

\subsection{Metropolis-adjusted Langevin algorithm}
\label{mala_section}

If $M$ is not a function of the state $\theta_k$, the
Metropolis-adjusted Langevin algorithm (MALA) arises as
\begin{align}
\mu(\theta_k, M, \epsilon) & = 
\theta_k+\frac{\epsilon^2}{2}M^{-1}\nabla\log{p(\theta_k)},
\label{mala_location}\\
\Sigma(M, \epsilon) & = \epsilon^2 M^{-1}.
\label{mala_covariance}
\end{align}
$M$ is known as the precondition matrix, see \cite{rob_stra__lan}. It is 
typically set to be the identity matrix $M=I$, in
which case MALA is defined in its conventional form.

MALA uses the gradient flow $\nabla\log{p(\theta)}$ to make proposals 
effectively. According to the 
theoretical analysis of \cite{rob_ros__opt}, the optimal scaling $\epsilon$ has 
been found to be the value of $\epsilon$ 
which yields a limiting acceptance rate of $57.4\%$ in high-dimensional 
parametric spaces (as $n\rightarrow\infty$).

\subsection{Manifold Langevin Monte Carlo}
\label{smmala_section}

It is possible to incorporate further geometric structure in the form of a 
position-dependent metric 
$M(\theta_k)$, see \cite{gir_cal__rie,xif_she_liv__lan}. The Langevin diffusion 
is defined on a 
Riemann manifold endowed by the metric $M(\theta_k)$.
At the $k$-th iteration of the associated manifold Metropolis-adjusted Langevin 
algorithm (MMALA), candidate states are drawn from a normal proposal density 
with location and covariance given by
\begin{align}
\mu(\theta_k, M(\theta), \epsilon)
&= 
\theta_k+\displaystyle\frac{\epsilon^2}{2}M^{-1}(\theta_k)\nabla\log{p(\theta_k)}
+
\epsilon^2\boldsymbol{\gamma}(\theta_k),
\label{mmala_location}\\
\Sigma(M(\theta_k), \epsilon)
&= \epsilon^2 M^{-1}(\theta_k),
\label{mmala_covariance}
\end{align}
where the $i$-th coordinate $\gamma_i(\theta_k)$ of 
$\boldsymbol{\gamma}(\theta_k)\in\mathbb{R}^{n}$ is
\begin{equation}
\label{mmala_gamma}
\begin{split}
\gamma_i(\theta_k) &=
\displaystyle\frac{1}{2}\sum_{j=1}^{n}
\frac{\partial M^{-1}_{ij}(\theta_k)}{\partial (\theta_{k})_{j}} \\
&=
-\displaystyle\frac{1}{2}
\sum_{j,h,l=1}^{n}
M^{-1}_{ih}(\theta_k)
\frac{\partial M_{hl}(\theta_k)}{\partial(\theta_{k})_{j}}
M^{-1}_{lj}(\theta_k).
\end{split}
\end{equation}
$(\theta_k)_j$, $M_{hl}(\theta_k)$ and $M^{-1}_{ij}(\theta_k)$ in 
\eqref{mmala_gamma} denote the respective
$j$-th coordinate of $\theta_k$, $(h, l)$-th element of $M(\theta_k)$ and $(i, 
j)$-th element of 
$M^{-1}(\theta_k)$.

As seen from \eqref{mmala_gamma}, the term $\boldsymbol{\gamma}(\theta_k)$ 
increases the computational complexity
of operations on
the proposal density for target densities with high number $n$ of dimensions or 
with high correlation between
parameters. To reduce the computational cost, $\boldsymbol{\gamma}(\theta_k)$ 
can be dropped from  
\eqref{mmala_location}, simplifying the proposal location to
\begin{equation}
\label{smmala_location}
\mu(\theta_k, M(\theta_k), \epsilon) =
\theta_k+
\frac{\epsilon^2}{2}M^{-1}(\theta_k)\nabla\log{p(\theta_k)}.
\end{equation}
The method with location and covariance specified by \eqref{smmala_location} 
and \eqref{mmala_covariance} is known as 
simplified Metropolis-adjusted Langevin algorithm (SMMALA).

The optimal stepsize $\epsilon$ for MMALA and SMMALA is empirically suggested 
by \cite{gir_cal__rie} to be set so as to 
obtain an acceptance rate of around $70\%$; this choice has not been analyzed 
yet from a theoretical standpoint analogously 
to the choice of scaling for MALA by \cite{rob_ros__opt}.

\subsection{Synopsis of Langevin Monte Carlo algorithms}

The proposal mechanisms of MALA, SMMALA and MMALA define valid MCMC methods 
that converge to the target distribution. In
fact, if the metric $M(\theta_k)$ is constant, then 
$\boldsymbol{\gamma}(\theta_k)$ vanishes, so
each of SMMALA and MMALA coincides with pre-conditioned MALA.

Each of the three Langevin Monte Carlo samplers incorporate different amount of 
local geometry in the proposal 
mechanism. MALA makes use only of the gradient of the log-target. SMMALA relies 
additionally on the position-specific 
metric tensor $M(\theta_k)$. MMALA takes into account the curvature of the 
manifold induced by $M(\theta_k)$, which 
implies calculating the metric derivatives $\partial 
M_{hl}(\theta_k)/\partial(\theta_{k})_j$ in \eqref{mmala_gamma}. 
Depending on 
the manifold and its curvature, the 
proposals of the three Langevin Monte Carlo 
samplers may exhibit varying efficiency in converging to the target 
distribution, thus leading to differences in 
effective sample size.

Increasing inclusion of local geometry in the proposal mechanism escalates 
computational complexity.
More specifically,
MALA, SMMALA and MMALA require computing first, second and third order 
derivatives of 
the target. It is thus clear that there is a trade-off between geometric 
exploitation of the target from within the 
proposal density and associated complexity of the proposal density,
which translates to a trade-off between effective sample size and 
runtime for MALA, SMMALA and MMALA.

\subsection{Basics of adaptive Metropolis}

Sampling methods that propose samples by using past values of the chain, thus 
breaking the Markov property, are 
referred to as adaptive Monte Carlo.
The first adaptive Metropolis (AM) algorithm, as introduced by 
\cite{haa_sak_tam__ana}, used a proposal
kernel based on the empirical covariance matrix of the whole chain at each 
iteration.

In its first appearance in \cite{haa_sak_tam__ana}, the AM algorithm was 
defined for target densities of bounded 
support to ensure convergence. \cite{rob_ros__exa} extended AM to work with 
targets of unbounded support 
by suggesting a mixture proposal density also based on the empirical covariance 
matrix of the whole chain at each 
iteration.

Several variations of the AM algorithm have appeared.
For example, \cite{haa_lai_mir__dra} combined adaptive Metropolis and delayed 
rejection methodology to construct the
delayed rejection adaptive Metropolis (DRAM) sampler, which outperforms its 
constituent methods in certain situations.
More recently, \cite{vih__rob} introduced the so-called robust
adaptive Metropolis (RAM) algorithm, which scales the empirical covariance 
matrix of the chain to yield a desired mean 
acceptance rate, typically $23.4\%$ in multidimensional settings.
A thorough overview of adaptive Metropolis methods can be found in 
\cite{gri_wal__ona}.

The ergodicity properties of adaptive Monte Carlo were studied in 
\cite{and_mou__ont,rob_ros__cou}.
In particular, \cite{rob_ros__cou} defined the diminishing adaptation and 
bounded convergence conditions. Their joint
satisfaction ensures asymptotic convergence to the target distribution. Thus, 
diminishing adaptation and 
bounded convergence provide a useful machinery for constructing adaptive Monte 
Carlo algorithms.

\subsection{The first adaptive Metropolis algorithm}

Consider a chain $\theta_{0:k}$ up to iteration $k$ generated by AM.
\cite{haa_sak_tam__ana} defined the proposal density of AM for the next 
candidate state 
$\theta^{\star}$ to be normal
with mean equal to the current point $\theta_{k}$ and covariance
$\beta S(\theta_{0:k})+\lambda\beta I$ based on the empirical covariance matrix
\begin{equation}
\label{am_empirical_cov}
S(\theta_{0:k})=
\frac{1}{k}\left(
\sum_{i=0}^{k}\theta_{i}\theta_{i}^{T}-
(k+1)\bar{\theta}_{k}\bar{\theta}_{k}^{T}
\right)
\end{equation}
of the whole history $\theta_{0:k}$.
The constant $\lambda$ is set to a small positive value to constrain the 
empirical covariance within 
$c_1 I \le S(\theta_{0:k}) \le c_2 I$ for some constants $c_1, c_2 > 0$, 
thereby ensuring convergence for 
target densities of bounded support. The tuning parameter $\beta$, which may 
only depend on 
dimension $n$, allows to scale the covariance of the proposal density.

It follows from \eqref{am_empirical_cov} that the empirical covariance at the 
$k$-th AM iteration calculates recursively as
\begin{equation}
\label{am_empirical_cov_rec}
kS(\theta_{0:k})
= \ (k-1)S(\theta_{0:k-1}) +\theta_{k} \theta_{k}^{T}
-(k+1)\bar{\theta}_{k} \bar{\theta}_{k}^{T}+
k\bar{\theta}_{k-1} \bar{\theta}_{k-1}^{T}.
\end{equation}
The sample mean in \eqref{am_empirical_cov_rec} is also calculable recursively 
according to
\begin{equation}
\label{am_mean_rec}
k\bar{\theta}_k=(k-1)\bar{\theta}_{k-1}+\theta_k.
\end{equation}
The recursive equations \eqref{am_empirical_cov_rec} and \eqref{am_mean_rec} 
make the empirical covariance and sample mean 
of the chain computationally tractable for an arbitrarily large chain length.

\subsection{Adaptive Metropolis with mixture proposal}

\cite{rob_ros__exa} initiated AM with proposal density at iteration $k$ given by
\begin{equation}
\label{am_rob_proposal}
a_{\theta_{0:k}}(\theta^{\star})=\ (1-\lambda)
\mathcal{N}(\theta^{\star}|
\theta_{k}, \beta S(\theta_{0:k}))
+
\lambda
\mathcal{N}(\theta^{\star}|
\theta_{k}, \gamma I).
\end{equation}
The acceptance probability for AM in \cite{rob_ros__exa} is
\begin{equation}
\label{r_a}
r_{a}(\theta_k,\theta^{\star}) :=
\min\left\{\frac{
	p(\theta^*)a_{\theta_{0:k}}(\theta_k)
}{
	p(\theta_k)a_{\theta_{0:k}}(\theta^*)
},1\right\}
\end{equation}
if $p(\theta_k)a_{\theta_{0:k}}(\theta^*)>0$, and 
$r_{a}(\theta_k,\theta^{\star}):=1$ otherwise.

The proposal kernel $A(\theta_{0:k},\cdot)$ corresponding to the mixture 
density $a_{\theta_{0:k}}$ of equation
\eqref{am_rob_proposal} is defined by setting
$E=\mathbb{R}^n$, $w=0$, $z=k$, $q_{\theta_{w:z}}=a_{\theta_{0:k}}$
and $\nu$ to be the Lebesgue measure in equation \eqref{Q}.

The first component of mixture $a_{\theta_{0:k}}$ is updated 
adaptively using the whole chain history $\theta_{0:k}$ in the calculation of 
the empirical covariance matrix 
$S(\theta_{0:k})$ and the second component is introduced to stabilize the 
algorithm.
A small positive constant $\lambda$ in \eqref{am_rob_proposal} ensures 
convergence for a large family of target 
densities of 
unbounded support, including those that are log-concave outside some arbitrary 
bounded region.
Two tuning parameters appear in \eqref{am_rob_proposal}, namely $\beta$ and 
$\gamma$, which may only depend on dimension
$n$. Each of these parameters allow to scale the covariance of the respective 
mixture component.

\section{Complexity bounds}

This section provides upper bounds for the computational complexity of 
geometric Langevin Monte Carlo. It concludes with a
short overview of the computational cost of adaptive Metropolis algorithms.

\subsection{Complexity bounds for differentiation}

Calculations associated with the proposal and target densities determine the 
computational cost of Langevin 
Monte Carlo methods. In brief, the main computational requirements include 
sampling from and evaluating the proposal, 
as well as evaluating the target and its derivatives.

According to \eqref{mala_covariance}, the proposal covariance $\epsilon^2 
M^{-1}$ of MALA is constant, therefore it is 
derivative-free. On the other hand, the proposal covariance $\epsilon^2 
M^{-1}(\theta_k)$ of SMMALA and MMALA 
introduced in \eqref{mmala_covariance} require to compute the negative Hessian 
of the log-target as the 
position-specific metric $M(\theta_k)$. As seen from \eqref{mala_location}, 
\eqref{smmala_location} and 
\eqref{mmala_location}, the proposal location of MALA, SMMALA and MMALA entail 
the gradient, Hessian and Hessian 
derivatives of the log-target. Hence, fully specifying the proposal density of 
MALA, SMMALA and MMALA 
requires up to first, second and third order derivatives of the log-target.


Assume an $n$-dimensional log-target 
$f:=\log{p}$
with complexity $\mathcal{O}(f)$.
Considering the 
highest order of log-target differentiation associated with each sampler, the 
incurring costs for 
target-related evaluations in MALA, SMMALA and MMALA grow as 
$\mathcal{O}(fn)$, 
$\mathcal{O}(fn^2)$ and $\mathcal{O}(f n^3)$, respectively.

\subsection{Complexity bounds for linear algebra}

Having computed the log-target and its derivatives, the Langevin Monte Carlo 
normal proposal \eqref{langevin_proposal} 
is available to sample from and evaluate. The major computational cost of 
evaluating and sampling from the normal 
proposal \eqref{langevin_proposal} is related to linear algebra calculations, 
namely to the inversion and Cholesky 
decomposition of the proposal covariance $\epsilon^2 M^{-1}(\theta_k)$.

Using the Cholesky approach,
a candidate state $\theta^{\star}$ can be sampled from the normal proposal 
\eqref{langevin_proposal} of a Langevin Monte Carlo method with mean 
$\mu(\theta_k, M, \epsilon)$ 
and covariance $\epsilon^2 M^{-1}(\theta_k)$ as
\begin{equation*}
\theta^{\star} =
\mu(\theta_k, M, \epsilon)+
\epsilon \left(M^{-0.5}(\theta_k)\right)^{'}\tau,
\end{equation*}
where $M^{-0.5}(\theta_k)$ denotes the Cholesky factorization
\begin{equation*}
\left(M^{-0.5}(\theta_k)\right)^{'}M^{-0.5}(\theta_k)=
M^{-1}(\theta_k)
\end{equation*}
and $\tau\sim\mathcal{N}(0, I)$, see \cite{chi_gre__und}.
So, sampling from the proposal has a complexity of
$\mathcal{O}(n^3)$
since 
it requires the inversion of 
metric $M(\theta_k)$ and the Cholesky decomposition of $M^{-1}(\theta_k)$, each 
of which are
$\mathcal{O}(n^3)$ operations.

The acceptance probability \eqref{langevin_acceptance} of SMMALA and MMALA 
requires 
to evaluate the normal proposal \eqref{langevin_proposal} at $\theta_k$. As it 
has become apparent,
a proposal density evaluation 
has a complexity of $\mathcal{O}(n^3)$ due to the inversion of $M(\theta_k)$ 
needed by the proposal covariance 
$\epsilon^2 M^{-1}(\theta_k)$.

To be precise, the normal proposal density \eqref{langevin_proposal} must be 
evaluated twice, once at $\theta_k$ 
and once at $\theta^{\star}$, due to its appearance both in the numerator and 
denominator of the acceptance
ratio \eqref{langevin_acceptance}. This implies twice the number of log-target 
differentiations and matrix inversions to
compute $M(\theta_k)$, $M(\theta^{\star})$ and their inverses. Nevertheless, 
the scaling factor of 
two can be omitted in big-$\mathcal{O}$ bounds since the numerics associated 
with the state $\theta_k$ 
are known from iteration $k-1$, with the exception of $M^{-0.5}(\theta_k)$.

So, SMMALA and MMALA require the Cholesky factorization 
$M^{-0.5}(\theta_k)$ to sample
$\theta^{\star}$ from the normal proposal $g_{\theta_k}$ and the matrix 
inverse  
$M^{-1}(\theta^{\star})$ to evaluate the normal proposal $g_{\theta^{\star}}$ at
$\theta$.
Consequently, the cost of linear algebra computations associated with the 
normal proposal density 
\eqref{langevin_proposal} is of order
$\mathcal{O}(n^3)$
for each of these two 
samplers.

The Coppersmith-Winograd algorithm has been
optimized to perform matrix multiplication and therefore matrix inversion in 
$\mathcal{O}(n^{2.373})$ time,
see \cite{dav_sto__imp}, \cite{wil__brea} and \cite{leg__pow}. Hence, 
optimized implementations of SMMALA and MMALA can evaluate and sample from 
their proposal densities in
time bounded by $n^3+n^{2.373}$,
thus reducing computational cost for smaller $n$.


MALA, as opposed to SMMALA and MMALA, relies on a constant preconditioning 
matrix $M$, therefore $M^{-1}$ and 
$M^{-0.5}$ are evaluated once and cached at the beginning of the simulation 
avoiding the
$\mathcal{O}(n^3)$
penalty.
Since $M$, $M^{-1}$ and $M^{-0.5}$ are 
cached upon initializing MALA,
the complexity of sampling and evaluating the normal proposal density of MALA 
is capped by
the quadratic form
$
(\theta^{\star}-\theta_k)'
M^{-1}
(\theta^{\star}-\theta_k)
$
at $\mathcal{O}(n^2)$.
If $M$ is set to be the identity matrix, then the quadratic form
$
(\theta^{\star}-\theta_k)'
M^{-1}
(\theta^{\star}-\theta_k)
$
simplifies to the inner product
$
\langle\theta^{\star}-\theta_k,
\theta^{\star}-\theta_k\rangle
$ and the complexity of linear algebra calculations associated with the MALA 
proposal density further reduces to order 
$\mathcal{O}(n)$.

\subsection{Differentiation versus linear algebra costs}
\label{diff_vs_lin_algebra}

Adding up the differentiation and linear algebra costs yields the order of 
complexity for Langevin Monte 
Carlo. Hence, it follows that MALA, SMMALA and MMALA run in 
$\mathcal{O}(\max{\{fn,n^2\}})$, 
$\mathcal{O}(\max{\{fn^2,n^3\}})$ and 
$\mathcal{O}(\max{\{fn^3,n^3\}})$ time, respectively.
The terms $fn$, $fn^2$ and $fn^3$
indicate the cost of differentiating log-target $f$,
while $n^2$ and $n^3$ indicate linear algebra costs.

As an example of a computationally cheap model,
consider an isotropic normal log-target $f$,
which has complexity $\mathcal{O}(f)=\mathcal{O}(n)$.
In this case,
the differentiation and linear algebra costs for MALA and SMMALA are comparable,
and differentiation is more costly than linear algebra for MMALA.

On the other hand, computationally expensive models yield
$\mathcal{O}(f)>>\mathcal{O}(n)$. For such
models, the cost of computations implicating the log-target is much higher than 
the cost of proposal-related calculations. 
In other words, if the log-target is of high complexity, then derivative 
calculations supersede linear algebra calculations, 
and this is why the computational cost of manifold MCMC algorithms tends to be 
reported as a function of the order of 
derivatives appearing in the algorithm. For instance, the complexity of SMMALA, 
which scales as 
$\mathcal{O}(\max{\{fn^2,n^3\}})$, can be simply written as 
$\mathcal{O}(fn^2)$ for a computationally intensive model.

An example of a computationally expensive model is a system of non-linear 
ordinary differential equations (ODEs), 
where each log-target calculation requires solving the ODE system numerically. 
It is then expected that the log-target and 
its derivative evaluations will dominate the cost of Langevin Monte Carlo 
simulations.



\subsection{Complexity bounds for adaptive Metropolis}

From this point forward, the term adaptive Metropolis (AM) will refer to the AM 
algorithm of \cite{rob_ros__exa} with 
mixture proposal \eqref{am_rob_proposal}, as interest is in targets of 
unbounded support. AM does not evaluate any 
target-related derivatives. In lieu of differentiation costs, the 
target-specific complexity of AM is of order 
$\mathcal{O}(f)$.

The components of mixture density \eqref{am_rob_proposal} are centered at the 
current state, and the empirical covariance 
of the adaptive component is computed recursively. Thus, 
fully specifying the AM proposal density is computationally trivial given the 
chain history.

Sampling from and evaluating the fully specified normal mixture 
\eqref{am_rob_proposal} of AM incurs the typical linear 
algebra computational costs encountered in Langevin Monte Carlo, namely a 
Cholesky decomposition and an inversion of the 
empirical covariance matrix $S(\theta_{0:k})$. So, linear algebra manipulations 
of the AM proposal amount to a 
complexity of order
$\mathcal{O}(n^3)$.

The recursive formula \eqref{am_empirical_cov_rec} allows to replace the 
Cholesky factorization of 
$S(\theta_{0:k})$ by two rank one updates and one rank one downdate, thus 
reducing the Cholesky runtime bound 
of AM from $\mathcal{O}(n^3)$ to
$\mathcal{O}(n^2)$
\cite{gill_gol_wal__met} 
and \cite{see__low} 
elaborate on low rank updates for Cholesky decomposition.

In total, the computational cost of AM
is
$\mathcal{O}(\max{\{f,n^3\}})$.
It reduces to
$\mathcal{O}(\max\{f, n^{2.373}\})$
if optimized algorithms are chosen to 
invert 
$S(\theta_{0:k})$ and if low rank updates are used for factorizing 
$S(\theta_{0:k})$.

For an isotropic normal log-target $f$, AM has a complexity of 
$\mathcal{O}(n^{2.373})$, so it is more 
costly than MALA and cheaper than SMMALA and MMALA. For expensive targets with 
complexity 
$\mathcal{O}(f)>>\mathcal{O}(n)$, AM runs in $\mathcal{O}(f)$ time, so 
it is cheaper
than MALA, SMMALA and MMALA.

\subsection{Summary of complexity bounds}

Table \ref{complexity_summary} shows the general complexity bounds
$\mathcal{O}(f)$ per step
of geometric Langevin Monte Carlo and of adaptive Metropolis
for any log-target $f$.
Moreover, the last two columns of table \ref{complexity_summary} 
show the complexity bounds for
relatively cheap targets of linear complexity $\mathcal{O}(f)=\mathcal{O}(n)$
and for expensive targets of complexity $\mathcal{O}(f)>>\mathcal{O}(n)$.

\begin{table}[tp]
	\centering
	{\tabulinesep=1.2mm
	\begin{tabu}{l|l|l|l}
		\hline
		\multicolumn{1}{c|}{\multirow{2}{*}{Method}} &
		\multicolumn{1}{c|}{\multirow{2}{*}{General $\mathcal{O}(f)$}} & 
		\multicolumn{2}{c}{Special cases of $\mathcal{O}(f)$} \\ \cline{3-4}
		& &
		$\mathcal{O}(f)=\mathcal{O}(n)$ &
		\multicolumn{1}{c}{$\mathcal{O}(f)>>\mathcal{O}(n)$} \\
		\hline
		MALA &
		$\mathcal{O}(\max{\{fn,n^2\}})$ &
		$\mathcal{O}(n^2)$ & $\mathcal{O}(fn)$ \\
		\hline
		SMMALA &
		$\mathcal{O}(\max{\{fn^2,n^3\}})$ &
		$\mathcal{O}(n^3)$ & $\mathcal{O}(fn^2)$ \\
		\hline
		MMALA &
		$\mathcal{O}(\max{\{fn^3,n^3\}})$ &
		$\mathcal{O}(n^4)$ & $\mathcal{O}(fn^3)$ \\
		\hline
		AM &
		$\mathcal{O}(\max\{f, n^{2.373}\})$ &
		$\mathcal{O}(n^{2.373})$ & $\mathcal{O}(f)$ \\
		\noalign{\vskip 0.5mm}\hline
	\end{tabu}
	}
	\caption{General complexity bounds per step of MALA, SMMALA, MMALA and AM samplers,
	and two special cases of a log-target $f$ with linear complexity $\mathcal{O}(f)=\mathcal{O}(n)$ and of expensive log-targets $f$ with complexity
	$\mathcal{O}(f)>>\mathcal{O}(n)$.
}
\label{complexity_summary}
\end{table}

For relatively cheap targets of linear complexity,
MALA has lower order of complexity than AM, which in turn 
has lower order of complexity than SMMALA and 
MMALA.
For expensive targets,
MALA, SMMALA, MMALA and AM share the same order of complexity 
$\mathcal{O}(f)$, with respective scaling factors $n$, $n^2$, 
$n^3$ and $1$.
These scaling factors are negligible for very expensive targets,
but they affect the total computational cost for a range of targets of modest 
to high complexity.

\section{Geometric adaptive Monte Carlo}

Manifold Langevin Monte Carlo pays a higher computational price than adaptive 
Metropolis to achieve increased effective 
sample size via geometric exploitation of the target. To get the best of both 
worlds, the goal is to construct a Monte 
Carlo sampler that attains fast mixing per step but with less cost per step.
Along these lines, the present paper introduces GAMC,
a hybrid sampling method that 
switches between expensive geometric 
Langevin Monte Carlo and cheap adaptive Metropolis updates.

\subsection{Sampling in random environment}

GAMC is defined as a discrete-time stochastic process 
$\{\theta_k\}$ in IID random environment.
The environment is a sequence $\{B_k\}$ of independent random variables 
admitting a Bernoulli distribution
with probability $s_k:=P(B_k=1)$.

Let $\tau_k$ be the last time before iteration $k$ that the geometric kernel 
was used, defined as the stopping time
\begin{equation*}
\tau_k:=
\begin{cases}
\underset{0\le i <k}{\max}\{i: B_i=1\} & \mbox{if such $i$ exists}, \\
\;\; 0 & \mbox{otherwise}.
\end{cases}
\end{equation*}
The sequence $\{\tau_k\}$ of stopping times induces a sequence of random 
proposal kernels
\begin{equation*}
Q_{k}(\theta_{\tau_k:k},\cdot) :=
\begin{cases}
A(\theta_{\tau_k:k},\cdot) & \mbox{if } B_k=0, \\
G(\theta_k,\cdot) & \mbox{if } B_k=1,
\end{cases}
\end{equation*}
switching between
adaptive proposal kernels $A(\theta_{\tau_k:k},\cdot)$ and
geometric proposal kernel $G(\theta_k,\cdot)$.
GAMC provides a general Monte Carlo sampling scheme,
which is instantiated depending on the choice of kernels 
$A(\theta_{\tau_k:k},\cdot)$ and $G(\theta_k,\cdot)$.

It is noted that the dimension $k-\tau_k+1$ of the first argument 
$\theta_{\tau_k:k}\in E^{k-\tau_k+1}$ in the definition of
random kernel $Q_k:E^{k-\tau_k+1}\times\mathcal{E}\rightarrow\mathbb{R}_{+}$ 
varies between iterations due to the random
stopping time $\tau_k$.

For every $\theta_{\tau_k:k}\in E^{k-\tau_k+1}$, \cite{kal__ran} ensures that
the Radon-Nikodym derivative $q_{\theta_{\tau_k:k}}$ of random measure 
$Q_{k}(\theta_{\tau_k:k},\cdot)$ exists almost surely.

Equation \eqref{Q} is linked to
the Radon-Nikodym derivative $q_{\theta_{\tau_k:k}}$ of random measure 
$Q_{k}(\theta_{\tau_k:k},\cdot)$
by setting $E=\mathbb{R}^n$, $w=\tau_k$, $z=k$, $q_{w:z}=q_{\tau_k:k}$ and 
$\nu$ to be the Lebesgue measure.

Using the proposal density $q_{\theta_{\tau_k:k}}$,
the Metropolis-Hastings acceptance probability at the $k$-th iteration of GAMC 
is set to
\begin{equation*}
r_{q}(\theta_k,\theta^{\star}) :=
\min\left\{\frac{
	p(\theta^*)q_{\theta_{\tau_k:k}}(\theta_k)
}{
	p(\theta_k)q_{\theta_{\tau_k:k}}(\theta^*)
},1\right\}
\end{equation*}
if $p(\theta_k)q_{\theta_{\tau_k:k}}(\theta^*)>0$, and 
$r_{q}(\theta_k,\theta^{\star}):=1$ otherwise.

The process $\{\theta_k\}$ can be constructed from kernels $\{Q_k\}$ by 
extending the Ionescu Tulcea
theorem (see \cite{nev__the}) to processes in random environment.

A framework for generating chains via random proposal kernels is discussed in
\cite{rob_ros__cou}.
Non-Markovian chains in random environment, such as the process $\{\theta_k\}$ 
constructed via $\{Q_k\}$,
have received less attention than adaptive Monte Carlo methods in the 
literature.

\subsection{Algorithmic formulation}

Algorithm \ref{mamala_algo} provides a pseudocode representation of the 
proposed GAMC sampler.
The sequence $\{s_k\}$ of probabilities is deterministic.
Section \ref{choice_of_schedule} provides a condition on $\{s_k\}$ that ensures 
convergence of the GAMC sampler.

\begin{algorithm}[t]
	\caption{GAMC}
	\label{mamala_algo}
	\begin{algorithmic}
		\For{$k = 0$ to $m-1$} \Comment{$m$: number of iterations}
		
		\State Sample $B_k \sim \mbox{Bernoulli}(s_k)$\\
		
		\State $\tau_k =
		\begin{cases}
		\underset{0\le i <k}{\max}\{i: B_i=1\} & \mbox{if such $i$ exists}\\
		\;\; 0 & \mbox{otherwise}
		\end{cases}$\\
		
		\If {$B_k = 0$} \Comment{Use adaptive kernel}
		\State $Q_k(\theta_{\tau_k:k}, \cdot) = A(\theta_{\tau_k:k}, \cdot)$
		\ElsIf {$B_k=1$} \Comment{Use geometric kernel}
		\State $Q_k(\theta_{\tau_k:k}, \cdot) = G(\theta_k, \cdot)$
		\EndIf\\
		
		\State Sample $u\sim\mathcal{U}(0, 1)$ \Comment{Uniform density 
		$\mathcal{U}(0, 1)$}\\
		
		\State Sample 
		$\theta^{*}
		\sim Q_k(\theta_{\tau_k:k},\cdot)
		$\\
		
		\State $r_{q}(\theta_k,\theta^{\star}) =
		\min\left\{\frac{
			p(\theta^*)q_{\theta_{\tau_k:k}}(\theta_k)
		}{
			p(\theta_k)q_{\theta_{\tau_k:k}}(\theta^*)
		},1\right\}
		$\\
		
		\If
		{
			$u<r_{q}(\theta_k,\theta^{\star})$
		}
		\State $\theta_{k+1}=\theta^{\star}$
		\Else
		\State $\theta_{k+1}=\theta_k$		
		\EndIf
		
		\EndFor
	\end{algorithmic}
\end{algorithm}

At its $k$-th iteration, GAMC
uses either AM proposal kernel $A(\theta_{\tau_k:k}, \cdot)$ dependent on the 
past $k-\tau_k+1$
states $\theta_{\tau_k:k}$
as determined by the stopping time $\tau_k$
or LMC proposal kernel $G(\theta_k,\cdot)$ dependent only on the current state 
$\theta_k$.

Algorithm \ref{mamala_algo} demonstrates that the proposal covariance is based 
on the position-specific metric $M$ 
whenever possible and falls back to the empirical covariance $S$ otherwise. 
Thus, $M$ initializes $S$, 
and the latter is recursively updated via \eqref{am_empirical_cov_rec} until 
the next geometric update re-initializes the 
empirical covariance.

\subsection{Convergence properties}

This section establishes the convergence properties of GAMC.
Recall that $s_k$ is the probability of picking the geometric kernel at the 
$k$-iteration of GAMC.

\begin{proposition}
	\label{prop_convergence}
	If $\sum_{k=0}^{\infty}s_k < \infty$,
	then the convergence properties of GAMC are solely determined by the 
	convergence
	properties of its AM counterpart.
\end{proposition}

\begin{proof}
	
	
	Due to the Borel-Cantelli lemma, the assumption $\sum_{k=0}^{\infty}s_k < 
	\infty$ implies that
	the GAMC proposal kernel $Q_k$ is set to the geometric proposal kernel $G$ 
	only a finite number of times almost  
	surely. Hence,
	if the AM algorithm based on the adaptive proposal kernel $A$ of GAMC is 
	ergodic or satisfies the weak law of
	large numbers, then so does the corresponding GAMC sampler almost surely.
\end{proof}

\begin{corollary}
	\label{cor_convergence}
	If $\sum_{k=0}^{\infty}s_k < \infty$ and the adaptive proposal kernel of GAMC 
	is specified via the mixture proposal density
	\eqref{am_rob_proposal}, then GAMC satisfies the weak law of large numbers.
\end{corollary}

\begin{proof}
	Using the notation of section \ref{background}, set $E=\mathbb{R}^n$,
	equipped with the Borel $\sigma$-algebra $\mathcal{E}=\sigma (\mathbb{R}^n)$.
	Let $p\colon \mathbb{R}^n\rightarrow\mathbb{R}_{+}$ be a possibly 
	unnormalized target 
	density and $\pi\colon\sigma(\mathbb{R}^n)\rightarrow\mathbb{R}_{+}$ the 
	associated target distribution
	\begin{equation*}
	\pi(B)=\int_{B}pd\nu,\quad B\in\sigma(\mathbb{R}^n),
	\end{equation*}
	where $\nu$ is the Lebesgue measure.
	
	The AM algorithm of \cite{rob_ros__exa},
	as defined by \eqref{am_rob_proposal} and \eqref{r_a}, satisfies the weak law 
	of large numbers.
	Hence, according to proposition \ref{prop_convergence}, any chain 
	$\{\theta_k\}$ generated by GAMC also satisfies
	\begin{equation*}
	\lim_{m\to\infty}\frac{1}{m}\sum_{k=0}^{m}h(\theta_k)=\int_{\mathbb{R}^n}hd\pi
	\end{equation*}
	in probability
	for any bounded function $h\colon\mathbb{R}^n\rightarrow\mathbb{R}$.
\end{proof}

For AM kernels satisfying a set of different conditions
(see \cite{sak_vih__erg,bai_rob_ros__ont,haa_sak_tam__ana}),
AM and consequently GAMC are ergodic.


\subsection{Choice of schedule for geometric steps}
\label{choice_of_schedule}

A design decision to make is how to set the sequence of probabilities 
$\left\{s_k\right\}$ of choosing geometric over
adaptive steps. 
The choice of $\left\{s_k\right\}$ affects the convergence properties and the 
computational complexity of GAMC.

One possibility is to make the frequency of geometric steps more pronounced in 
early transient phases of the chain and 
let the computationally cheaper adaptive kernel take over asymptotically in 
late stationary phases.
This possibility is confined by the requirement of convergence, which in turn 
can be fulfilled by the condition
$\sum_{k=0}^{\infty}s_k < \infty$ of proposition \ref{prop_convergence}.

An example of a sequence of probabilities $\left\{s_k\right\}$ that conform to 
these practical guidelines and convergence
requirements is
\begin{equation}
\label{exp_schedule}
s_k=e^{-rk},
\end{equation}
where $r$ is a positive-valued tuning parameter.
Larger values of $r$ in \eqref{exp_schedule} yield 
faster reduction in the probability of using the geometric kernel.

The probabilities $\left\{s_k\right\}$ of GAMC play an analogous role as 
temperature in simulated annealing.
Thereby, $\left\{s_k\right\}$ can be thought as a schedule for regulating the 
choice of proposal kernel.
There is a rich literature on cooling schedules for simulated annealing
\cite{kir_gel_vec__opt,haj__coo,loc__sim,nou_and__aco,mar_sie__aco},
some of which can be employed as $\left\{s_k\right\}$.

In this paper, GAMC is equipped with the exponential schedule 
\eqref{exp_schedule}. Under schedule 
\eqref{exp_schedule}, 
GAMC and AM share similar convergence properties and complexity bounds 
asymptotically.
Yet GAMC has faster mixing per step 
than AM due to exploitation of local geometric information in early phases of 
the chain. The tuning parameter $r$ in
\eqref{exp_schedule} regulates the frequency of geometric steps and therefore 
the ratio of mixing per step and 
computational cost per step.

\subsection{Expected complexity}

The concept of complexity carries three meanings in the context of 
MCMC. Firstly, MCMC samplers need to be tuned so 
as to achieve a balance between proposing large enough jumps and ensuring that 
a reasonable proportion of jumps are accepted.
By way of illustration, MALA attains its optimal acceptance rate of $57.4\%$ as 
$n\rightarrow\infty$ by setting its
drift step $\epsilon$ to be in the vicinity of $n^{-1/3}$. Because of this, it 
is said that the algorithmic
efficiency of MALA scales $\mathcal{O}(n^{1/3})$ as the number $n$ of 
parameters increases.

Secondly, the quality of MCMC methods depends on their rate of mixing per step. 
Along these lines, the effective sample size 
(ESS) is used for quantifying the mixing properties of an MCMC method. The ESS 
of a chain of length $m$ is interpreted as 
the number of samples in the chain bearing the same amount of variance as the 
one found in $m$ independent samples.

A third criterion for assessing MCMC algorithms is their computational cost per 
step. This criterion corresponds to the 
ordinary concept of algorithmic complexity, as it entails a count of numerical 
operations performed by an MCMC algorithm. To 
give an example, the computational complexity of MALA with an identity 
preconditioning matrix for an isotropic normal target is of order 
$\mathcal{O}(n^2)$, as explained in section \ref{diff_vs_lin_algebra}.

Of these three indicators of complexity, ESS and computational runtime are the 
ones typically used for understanding the 
applicability of MCMC methods. To get a single-number summary, the ratio of ESS 
over runtime is usually employed.

The present section states the expected complexity per step of GAMC given the 
selected length $m$ of simulation, while 
section \ref{sim_study} provides an empirical assessment of GAMC via its ESS 
and CPU runtime.

\begin{proposition}
	Denote by $c_{g}$ and $c_{a}$ the computational complexities
	per geometric and adaptive Monte Carlo step of GAMC, respectively.
	The expected complexity per step of GAMC for generating an $m$-length chain is
	\begin{equation}
	\label{eq:c}
	\left(\frac{1}{m}\sum_{k=0}^{m-1}s_k\right) 
	c_{g}+\left(1-\frac{1}{m}\sum_{k=0}^{m-1}s_k\right) c_{a}.
	\end{equation}
\end{proposition}

\begin{proof}
	The expected number of geometric steps equals
	\begin{equation*}
	E\left(\sum_{k=0}^{m-1}B_k\right)=\sum_{k=0}^{m-1}E(B_k)=\sum_{k=0}^{m-1}s_k,
	\end{equation*}
	whence the conclusion follows directly.
\end{proof}

\begin{corollary}
	If the exponential schedule \eqref{exp_schedule} is used for regulating the 
	choice of proposal kernel, then the 
	expected complexity per step of GAMC for generating an $m$-length chain 
	expresses as
	\begin{equation}
	\label{eq:c_exp}
	\dfrac{1-e^{-rm}}{m(1-e^{-r})}c_{g}+
	\left(1-\dfrac{1-e^{-rm}}{m(1-e^{-r})}\right)c_{a}.
	\end{equation}
\end{corollary}

\begin{proof}
	Under the exponential schedule \eqref{exp_schedule}, observe that
	\begin{equation*}
	\sum_{k=0}^{m-1}s_k=\sum_{k=0}^{m-1}e^{-rk}=\frac{1-e^{-rm}}{1-e^{-r}},
	\end{equation*}
	whence \eqref{eq:c} yields \eqref{eq:c_exp}.
\end{proof}

\begin{corollary}
	\label{corol_asym_complexity}
	As the number $m$ of iterations gets large ($m\rightarrow\infty$), the 
	expected complexity per step of GAMC under the 
	exponential schedule \eqref{exp_schedule} reduces to the complexity of its AM 
	counterpart.
\end{corollary}

\begin{proof}
	Since
	\begin{equation*}
	\lim_{m\rightarrow\infty}\frac{1-e^{-rm}}{m(1-e^{-r})}= 0,
	\end{equation*}
	the bound \eqref{eq:c_exp} diminishes asymptotically to 
	$\mathcal{O}\left(c_{a}\right)$.
\end{proof}

As an example, consider the GAMC sampler
with AM proposal kernels induced by \eqref{am_rob_proposal} and
SMMALA proposal kernel as induced by \eqref{langevin_proposal}, 
\eqref{smmala_location} and \eqref{mmala_covariance}.
For such a configuration of GAMC,
as seen from table \ref{complexity_summary},
expensive targets with complexity 
$\mathcal{O}\left(f\right)>>\mathcal{O}\left(n\right)$ are associated with 
complexities
$c_g=$ $\mathcal{O}\left(fn^2\right)$ and 
$c_a=\mathcal{O}\left(f\right)$ in \eqref{eq:c_exp}.
So, the expected complexity per step of GAMC for generating an $m$-length chain 
is
\begin{equation}
\label{eq:c_exp_ex}
\mathcal{O}\left(
\max{\left\{
\dfrac{1-e^{-rm}}{m(1-e^{-r})}fn^2,
\left(1-\dfrac{1-e^{-rm}}{m(1-e^{-r})}\right)f
\right\}}
\right)
\end{equation}
for expensive targets,
which is bounded below by the AM complexity of $\mathcal{O}\left(f\right)$ 
and above by the SMMALA 
complexity of $\mathcal{O}\left(fn^2\right)$. For instance, setting 
$m=10^5$ and $r=10/m=10^{-4}$ 
in \eqref{eq:c_exp_ex} yields an expected complexity per step of GAMC equal to 
$\mathcal{O}\left(\max{\{0.1fn^2,0.9f\}}\right)$.
For increasing number $m$ of iterations, the expected complexity
per step of GAMC in \eqref{eq:c_exp_ex} tends to the lower bound 
$\mathcal{O}\left(f\right)$ of AM complexity
(see corollary \ref{corol_asym_complexity}).

More 
generally, the convergence properties and computational complexity of GAMC are 
determined asymptotically by the AM proposal 
kernel used in GAMC.
Despite the shared asymptotic properties of GAMC and AM,
the SMMALA steps in early transient phases of GAMC provide an improvement in 
mixing over AM. For example, setting 
$m=10^5$ and $r=10^{-4}$ in \eqref{eq:c_exp_ex} produces an expected $10\%$ of 
SMMALA steps, which is a potentially 
sufficient perturbation in early stages of parameter space exploration so as to 
move to target modes of higher probability 
mass.

\subsection{Analytically intractable geometric steps}

In practice, challenges in the implementation of manifold MCMC algorithms might 
raise additional computational implications.
In particular, two notoriously recurring issues relate to the Cholesky 
decomposition of metric $M^{-1}$ 
and to the calculation of up to third order derivatives of $M$.

Various factors, such as finite-precision floating point arithmetic, can lead 
to an indefinite proposal
covariance matrix $\epsilon^2 M^{-1}$. This in turn breaks the Cholesky 
factorization of 
$\epsilon^2 M^{-1}$. Several research avenues have introduced alternative 
positive definite 
approximations of indefinite matrices 
\cite{hig__com01,hig__com02,hig_stra__and} and approximate Riemann manifold 
metric 
choices \cite{bet__age,hou__hes,kle__ada}, which offer proxies for an 
indefinite covariance matrix
$\epsilon^2 M^{-1}$.

Non-trivial models can render the analytic derivation of log-target derivatives 
impossible or impractical. Automatic
differentiation (AD), a computationally driven research activity that has 
evolved since the mid 1950's, helps compute 
derivatives in a numerically exact way. Indeed, \cite{gri__ona} has shown that 
AD is backward stable in the sense of 
\cite{wil__mod}. Thus, small perturbations of the original function due to 
machine precision still yield accurate 
derivatives calculated via AD.

There are different methods of automatic differentiation that mainly differ in 
the way they traverse the chain rule; reverse 
mode AD is better suited for functions $h:\mathbb{R}^n\rightarrow\mathbb{R}$, 
in contrast to forward mode AD that is more 
suitable for functions $h:\mathbb{R}\rightarrow\mathbb{R}^m$ 
\cite{gri_wal__eva}. Consequently, reverse mode AD is 
utilized for computing derivatives of probability densities, and finds use in 
statistical inference. Reverse mode AD is 
not worse than that of the respective analytical derivatives of a target 
density in terms of complexity, but it poses high
memory requirements. Hybrid AD procedures combining elements of forward and 
backward propagation of derivatives can be 
constructed for achieving a compromise between execution time and memory usage 
when differentiating functions of the form
$h:\mathbb{R}^n\rightarrow\mathbb{R}^m$.

\section{Simulation study}
\label{sim_study}

In this simulation study, GAMC uses the exponential schedule 
\eqref{exp_schedule} to switch randomly between the AM kernel 
specified via mixture \eqref{am_rob_proposal} and the SMMALA kernel.
GAMC is compared empirically against its AM and SMMALA counterparts, as well as 
against MALA, in terms of mixing and cost 
per step via three examples. The examples revolve around a multivariate 
t-distribution with correlated coordinates, and two 
planetary systems, one with a single planet and one with two planets. 

Ten chains are generated by each sampler for each example. $1.1\times 10^5$ 
iterations are run for the realization of each 
chain, of which the first $10^4$ are discarded as burn-in, so $m=10^5$ samples 
per chain are retained in subsequent 
descriptive statistics.

To assess the quality of mixing of a sampler, the ESS of each chain generated 
by the sampler is computed.
The ESS of a coordinate of the vector $\theta\in\mathbb{R}^{n}$ of parameters 
is defined as
$\mbox{ESS}_m =
n_m\hat{\sigma}^2_{\mbox{\tiny IID}}/\hat{\sigma}^2_{\mbox{\tiny MC}}$,
where $\hat{\sigma}^2_{\mbox{\tiny IID}}$ and $\hat{\sigma}^2_{\mbox{\tiny 
MC}}$ denote the estimated ordinary and Monte 
Carlo variance of the chain associated with the parameter coordinate.
$\hat{\sigma}^2_{\mbox{\tiny MC}}$
is calculated using the
initial monotone sequence estimator of 
\cite{gey__pra}.

To assess the computational cost of a sampler, the CPU runtime of each chain 
generated by the sampler is recorded. The ESS 
per parameter coordinate and CPU runtime are reported by taking their 
respective means across the set of ten simulated 
chains.

The computational efficiency of a sampler is defined as the ratio of minimum 
ESS among all $n$ parameter 
coordinates over CPU runtime. Finally, the speed-up of a sampler relatively to 
MALA is set to be the ratio of MALA 
efficiency over the efficiency of the sampler.

The hyperparameter values $\lambda=0.01$, $\gamma=0.001$ in 
\eqref{am_rob_proposal} and $r=10/m=10^{-4}$ in \eqref{exp_schedule} are used 
across all simulations, as the result of
empirical tuning.
On the other hand, hyperparameter $\beta$ in \eqref{am_rob_proposal} is set via 
empirical tuning in the burn-in phase of 
each chain separately.
Automatic differentiation and the SoftAbs 
approximation of $\epsilon^2 M^{-1}$ \cite{bet__age} are used in all three 
examples.

Table \ref{speedup_tables} provides numerical summaries,
while figures \ref{acf_and_mean_figs} and \ref{traceplot_figs} display
visual summaries for the three examples.
Table \ref{speedup_tables} gathers the ESS, runtime, efficiency and speed-up,
as these arise 
after averaging across the ten simulated chains per sampler.
Figures \ref{acf_and_mean_figs} and \ref{traceplot_figs}
visualize
the running mean, autocorrelation and trace of one specimen 
chain per sampler out of the ten simulated chains.

A package, called \texttt{GAMCSampler}, implements GAMC using the 
\texttt{Julia} programming 
language. \texttt{GAMCSampler} is based on \texttt{Klara}, a package for MCMC 
inference written in \texttt{Julia} by one 
of the three authors.
\texttt{GAMCSampler} is open-source 
software available at
\url{https://github.com/papamarkou/GAMCSampler.jl}
along with the three examples of this paper. The packages \texttt{ForwardDiff} 
\cite{rev_lub_pap__for} and
\texttt{ReverseDiff}, which are also 
written in \texttt{Julia}, provide forward and reverse-mode automatic 
differentiation functionality.
Among these two AD packages, \texttt{ForwardDiff} has been put into practice in 
the simulations due to being more mature
and more optimized than \texttt{ReverseDiff}.

\begin{table}[tp]
	\centering
	\begin{tabular}{@{} l|r|rrrr|r|r|r @{}}
		\hline\noalign{\vskip 0.5mm}
		\multicolumn{9}{c}{Student's t-distribution} \\
		\noalign{\vskip 0.5mm}\hline\noalign{\vskip 0.5mm}
		\multicolumn{1}{c|}{\multirow{2}{*}{Method}} &
		\multicolumn{1}{c|}{\multirow{2}{*}{AR}} &
		\multicolumn{4}{c|}{ESS} &
		\multicolumn{1}{c|}{\multirow{2}{*}{t}} &
		\multicolumn{1}{c|}{\multirow{2}{*}{ESS/t}} &
		\multicolumn{1}{c}{\multirow{2}{*}{Speed}} \\
		& & \multicolumn{1}{c}{min} & \multicolumn{1}{c}{mean} & 
		\multicolumn{1}{c}{median} & \multicolumn{1}{c|}{max} & & & \\
		\noalign{\vskip 0.5mm}\hline\noalign{\vskip 0.5mm}
		MALA & 0.59 & 135 & 159 & 145 & 234 & 9.33 & 14.52 & 1.00\\
		AM & 0.03 & 85 & 118 & 117 & 155 & 17.01 & 5.03 & 0.35\\
		SMMALA & 0.71 & 74 & 87 & 86 & 96 & 143.63 & 0.52 & 0.04\\
		GAMC & 0.26 & \textbf{1471} & \textbf{1558} & \textbf{1560} & \textbf{1629} 
		& 31.81 & 46.23 & \textbf{3.18}\\
		\noalign{\vskip 0.5mm}\hline\noalign{\vskip 0.5mm}
		\multicolumn{9}{c}{One-planet system} \\
		\noalign{\vskip 0.5mm}\hline\noalign{\vskip 0.5mm}
		\multicolumn{1}{c|}{\multirow{2}{*}{Method}} &
		\multicolumn{1}{c|}{\multirow{2}{*}{AR}} &
		\multicolumn{4}{c|}{ESS} &
		\multicolumn{1}{c|}{\multirow{2}{*}{t}} &
		\multicolumn{1}{c|}{\multirow{2}{*}{ESS/t}} &
		\multicolumn{1}{c}{\multirow{2}{*}{Speed}} \\
		& & \multicolumn{1}{c}{min} & \multicolumn{1}{c}{mean} & 
		\multicolumn{1}{c}{median} & \multicolumn{1}{c|}{max} & & & \\
		\noalign{\vskip 0.5mm}\hline\noalign{\vskip 0.5mm}
		MALA & 0.55 & 4 & 76 & 18 & 394 & 57.03 & 0.07 & 1.00\\
		AM & 0.08 & 1230 & 1397 & 1279 & 2035 & 48.84 & 25.18 & \textbf{378.50}\\
		SMMALA & 0.71 & 464 & 597 & 646 & 658 & 208.46 & 2.23 & 33.45\\
		GAMC & 0.30 & \textbf{1260} & \textbf{2113} & \textbf{2151} & \textbf{3032} 
		& 76.80 & 16.41 & 246.59\\
		\noalign{\vskip 0.5mm}\hline\noalign{\vskip 0.5mm}
		\multicolumn{9}{c}{Two-planet system} \\
		\noalign{\vskip 0.5mm}\hline\noalign{\vskip 0.5mm}
		\multicolumn{1}{c|}{\multirow{2}{*}{Method}} &
		\multicolumn{1}{c|}{\multirow{2}{*}{AR}} &
		\multicolumn{4}{c|}{ESS} &
		\multicolumn{1}{c|}{\multirow{2}{*}{t}} &
		\multicolumn{1}{c|}{\multirow{2}{*}{ESS/t}} &
		\multicolumn{1}{c}{\multirow{2}{*}{Speed}} \\
		& & \multicolumn{1}{c}{min} & \multicolumn{1}{c}{mean} & 
		\multicolumn{1}{c}{median} & \multicolumn{1}{c|}{max} & & & \\
		\noalign{\vskip 0.5mm}\hline\noalign{\vskip 0.5mm}
		MALA & 0.59 & 5 & 52 & 10 & 377 & 219.31 & 0.02 & 1.00\\
		AM & 0.01 & 18 & 84 & 82 & 248 & 81.24 & 0.22 & 9.05\\
		SMMALA & 0.70 & 53 & 104 & 100 & 161 & 1606.92 & 0.03 & 1.37\\
		GAMC & 0.32 & \textbf{210} & \textbf{561} & \textbf{486} & \textbf{1110} & 
		328.08 & 0.64 & \textbf{26.39}\\
		\noalign{\vskip 0.5mm}\hline
	\end{tabular}
	\caption{Comparison of sampling efficacy between MALA, AM, SMMALA and GAMC 
	for the t-distribution,
	one-planet and two-planet system. AR: acceptance rate; ESS: effective 
	sample size; t: CPU runtime in seconds; ESS/t:
	smaller ESS across model parameters divided by runtime; Speed: ratio of 
	ESS/t for MALA over ESS/t for each other sampler.
	All tabulated numbers have been rounded to the second decimal place, apart 
	from effective sample sizes, which have been 
	rounded to the nearest integer.
	The minimum, mean,
	median and maximum ESS across the effective sample sizes of the twenty, six 
	and eleven parameters
	(associated with the respective t-distribution, one-planet and two-planet 
	system) are 
	displayed.
}
\label{speedup_tables}
\end{table}

\begin{figure}[tp]
	\centering
	\subfloat[Running mean (t-distribution)]{
		\label{acf_and_mean_figs_c}
		\includegraphics[width=2.3686in]{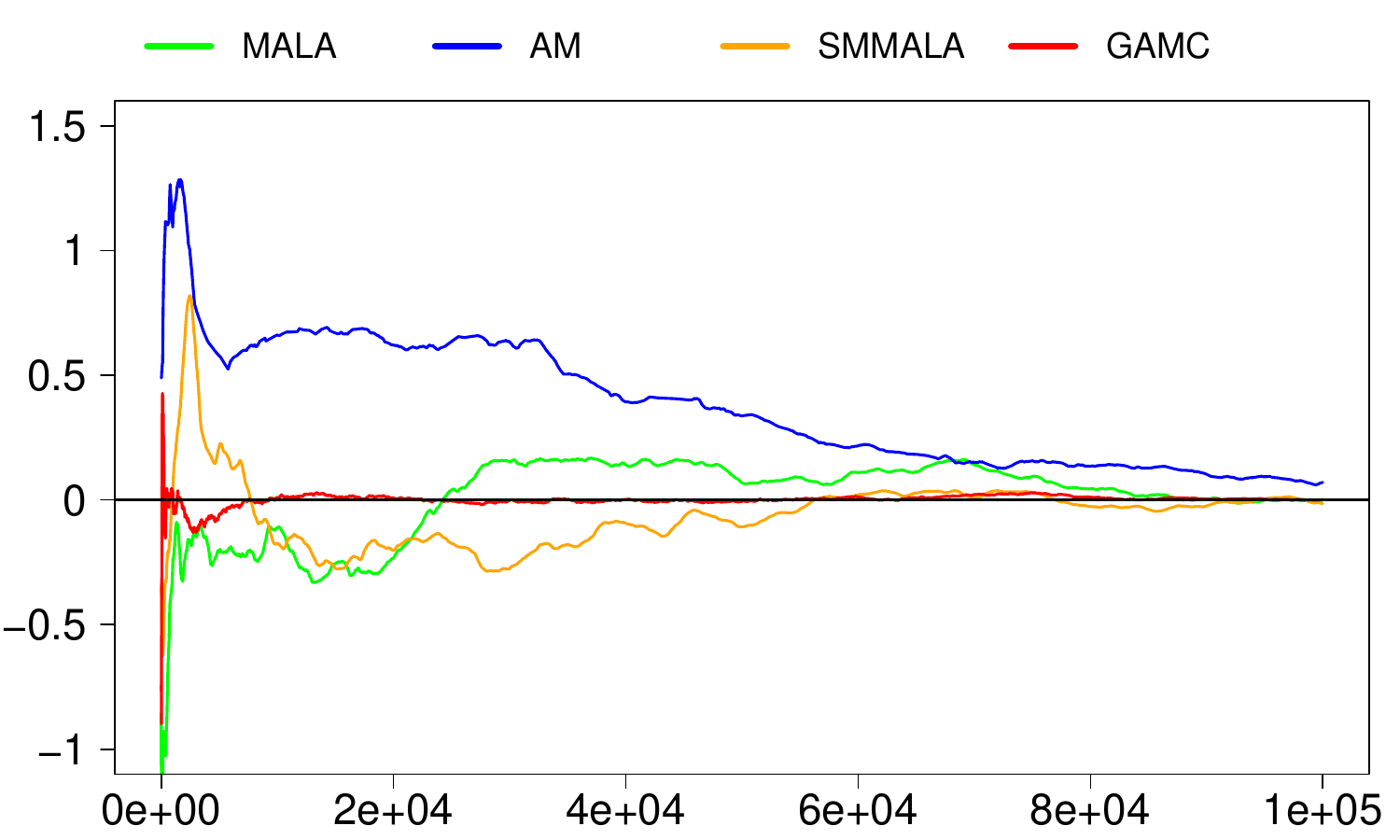}
	} 
	\subfloat[Autocorrelation (t-distribution)]{
		\label{acf_and_mean_figs_d}
		\includegraphics[width=2.3686in]{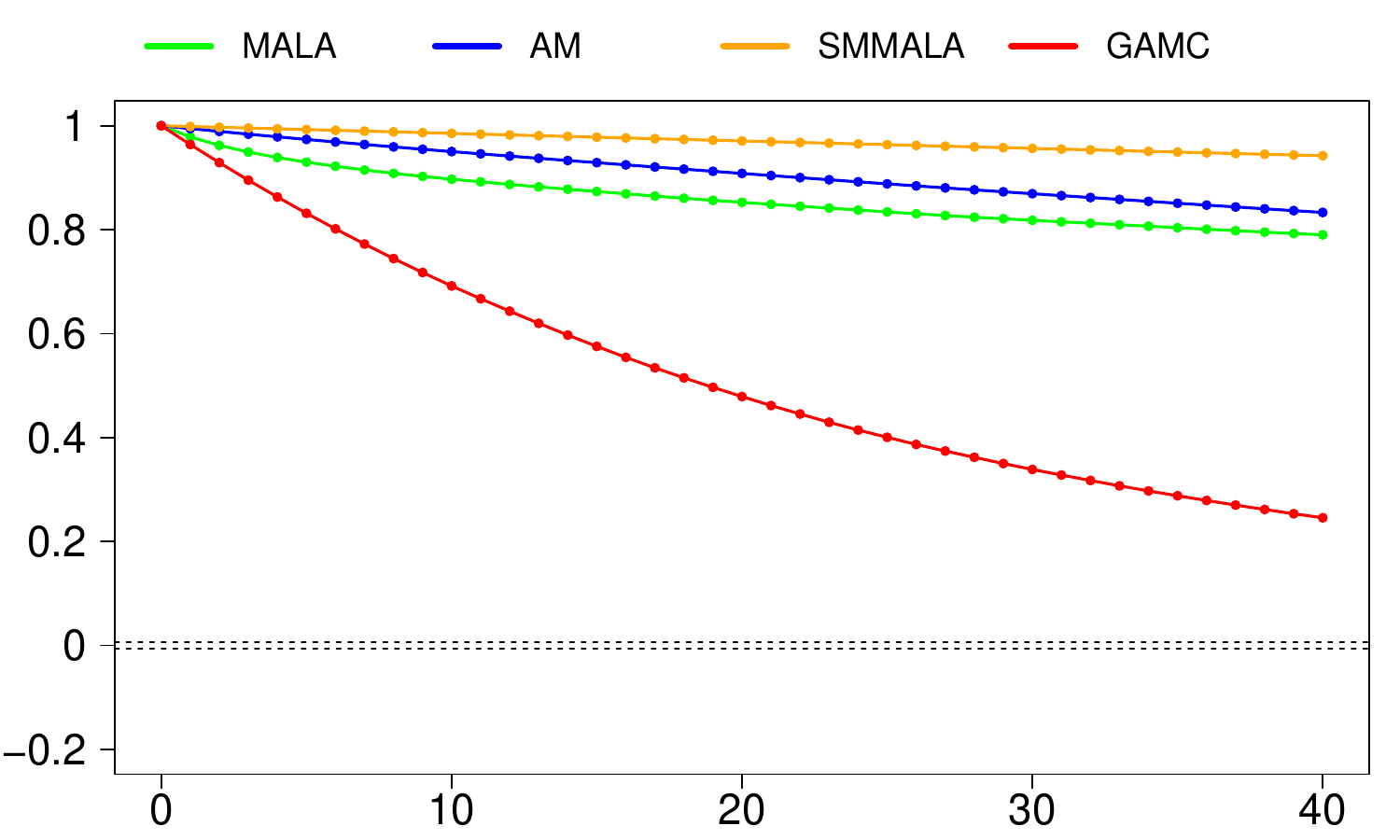}
	} \\
	\subfloat[Running mean (one-planet system)]{
		\label{acf_and_mean_figs_e}
		\includegraphics[width=2.3686in]{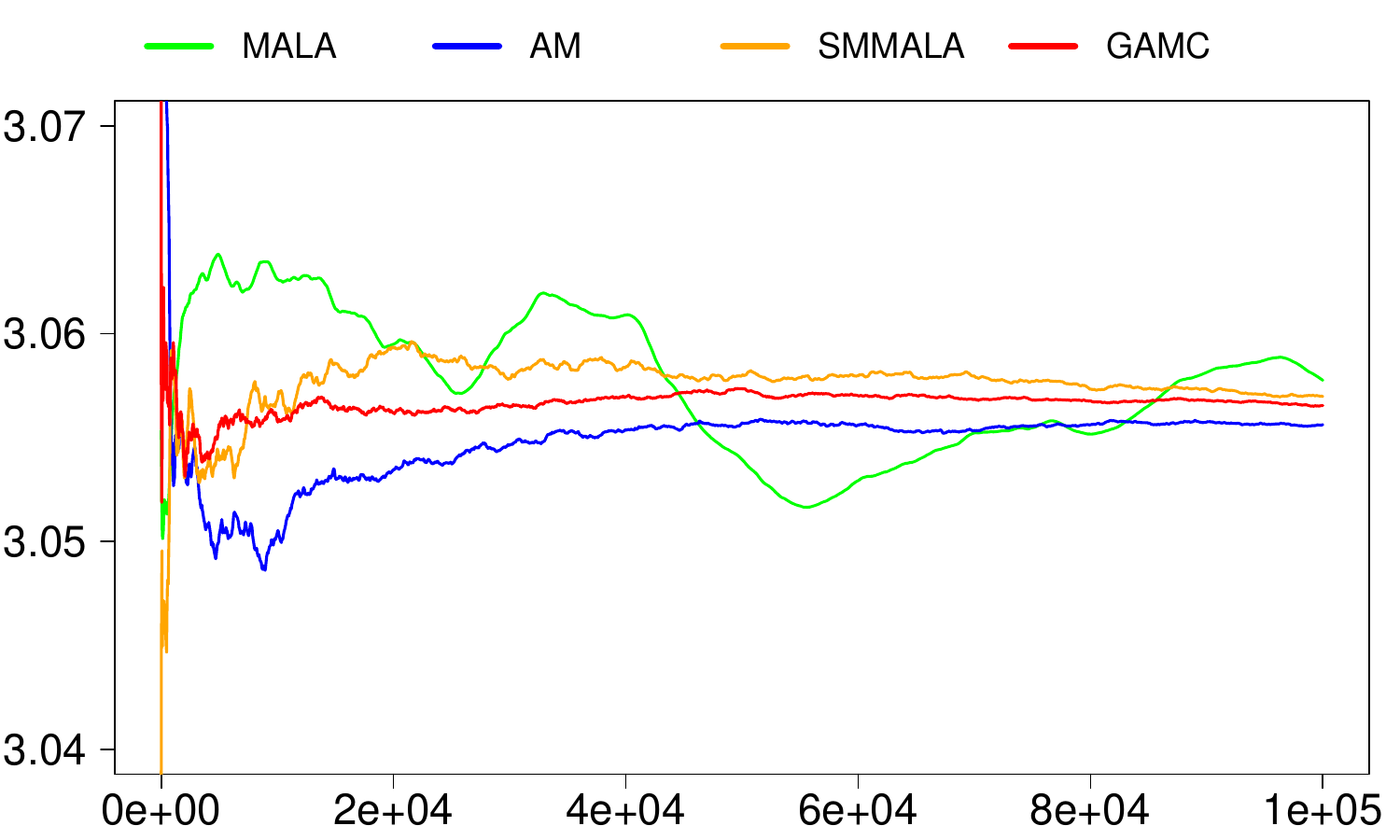}
	} 
	\subfloat[Autocorrelation (one-planet system)]{
		\label{acf_and_mean_figs_f}
		\includegraphics[width=2.3686in]{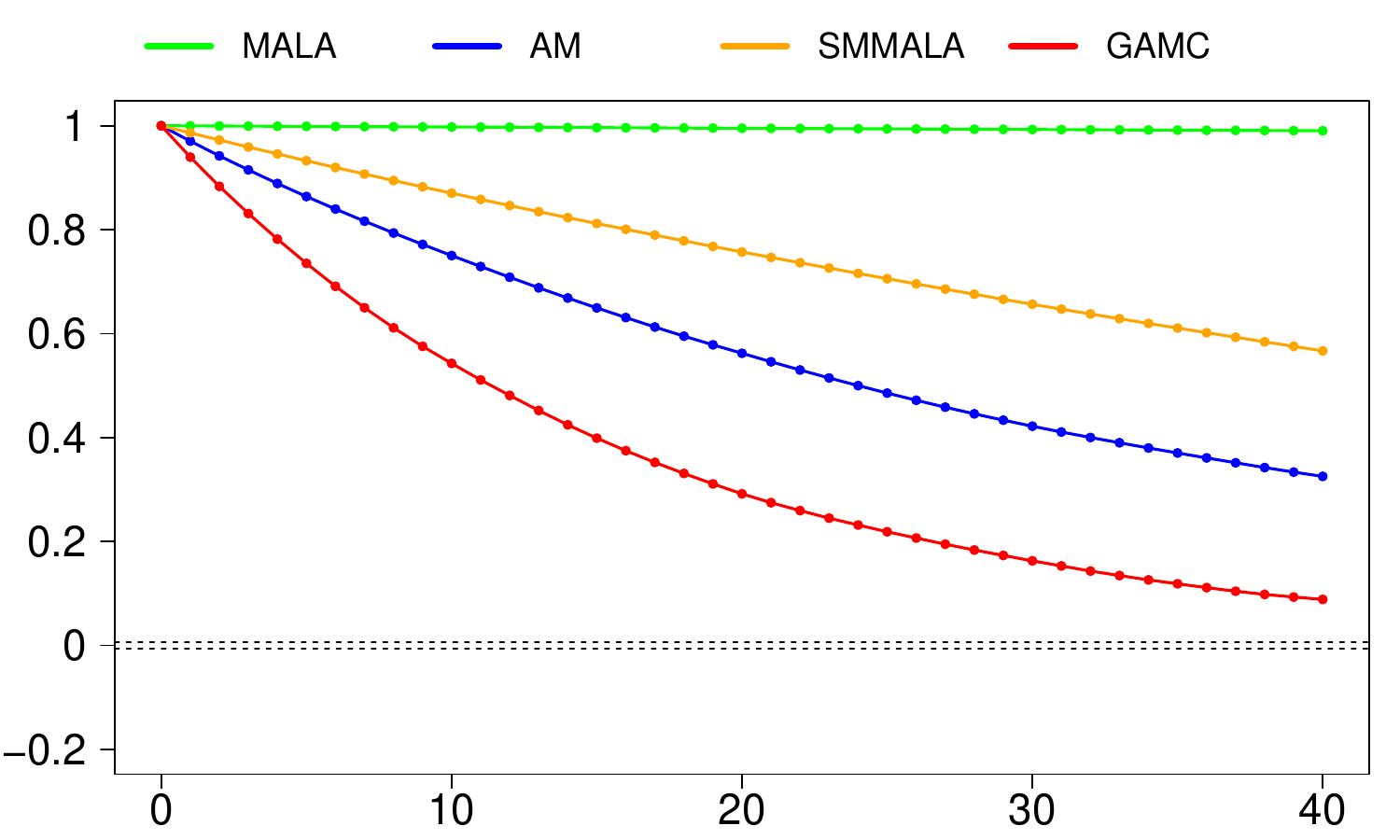}
	} \\
	\subfloat[Running mean (two-planet system)]{
		\label{acf_and_mean_figs_g}
		\includegraphics[width=2.3686in]{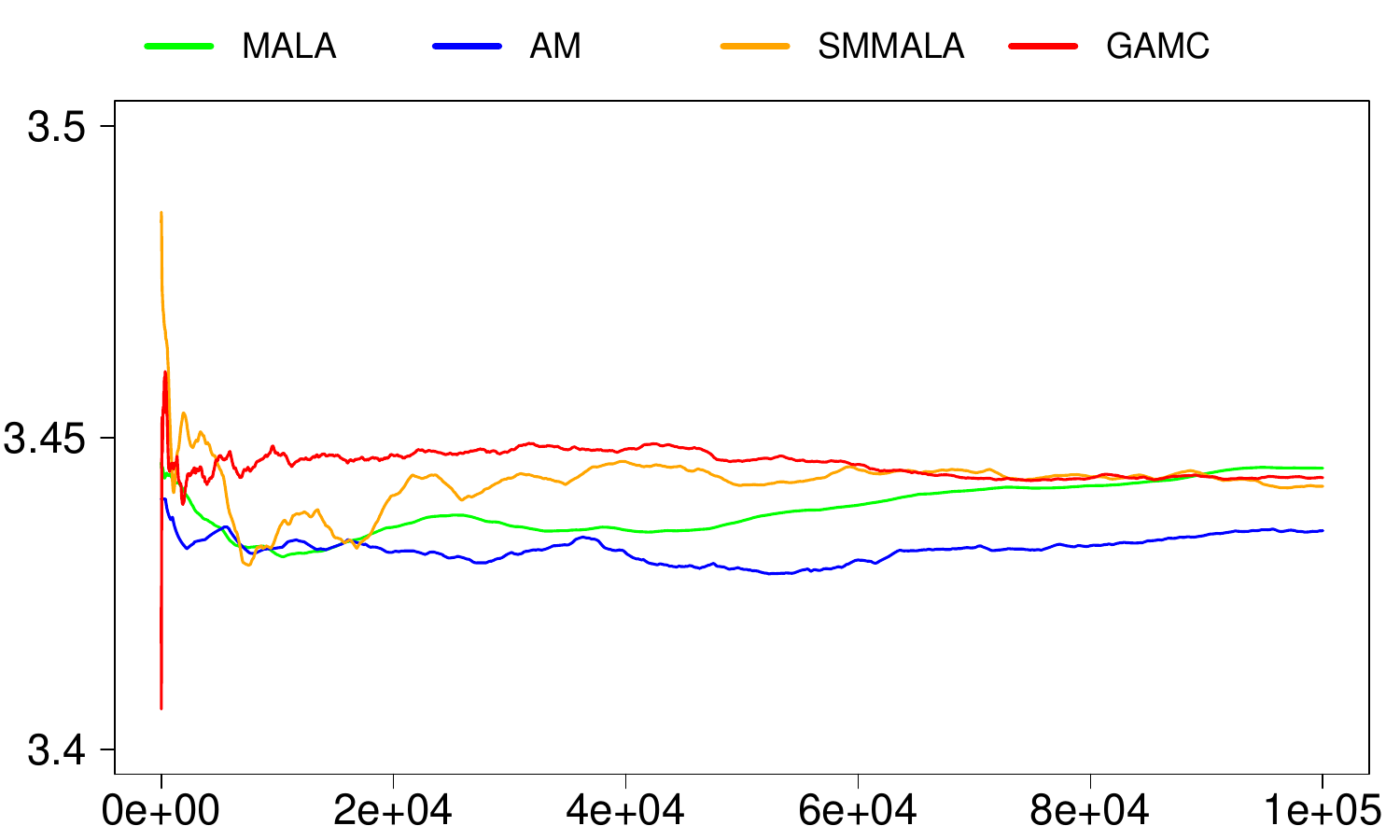}
	} 
	\subfloat[Autocorrelation (two-planet system)]{
		\label{acf_and_mean_figs_h}
		\includegraphics[width=2.3686in]{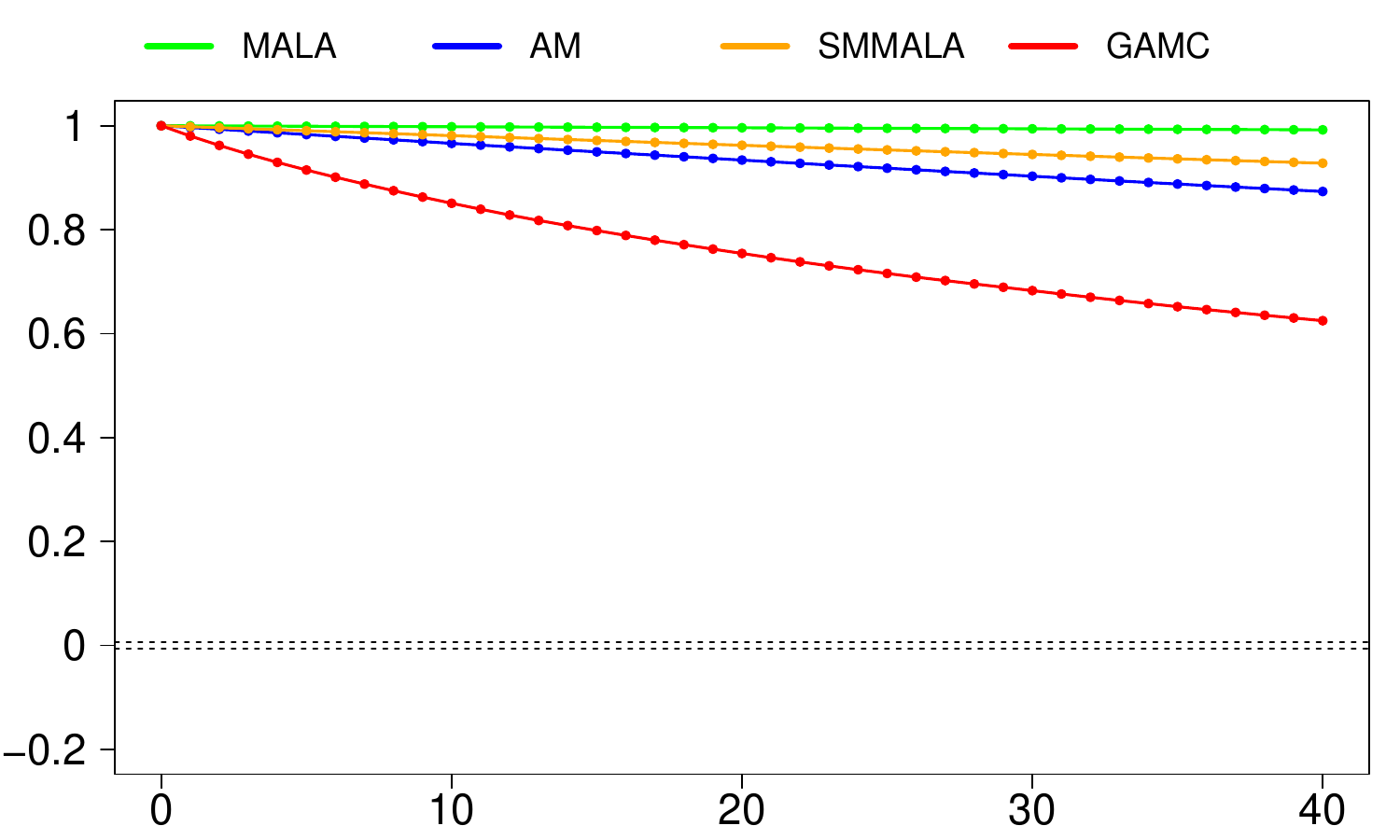}
	} \\
	\caption{Overlaid running means
	as a function of Monte Carlo iteration
	and overlaid linear autocorrelations of 
		single chains corresponding to one of the 
		twenty, six and eleven parameters of the respective t-distribution, 
		one-planet and two-planet system. 
		The black horizontal line in the t-distribution running mean plot 
		represents the true mode.}
	\label{acf_and_mean_figs}
\end{figure}


\begin{figure}[tp]
	\centering
	\subfloat[MALA traceplot (t-distribution)]{
		\label{traceplot_figs_a}
		\includegraphics[width=2.3686in]{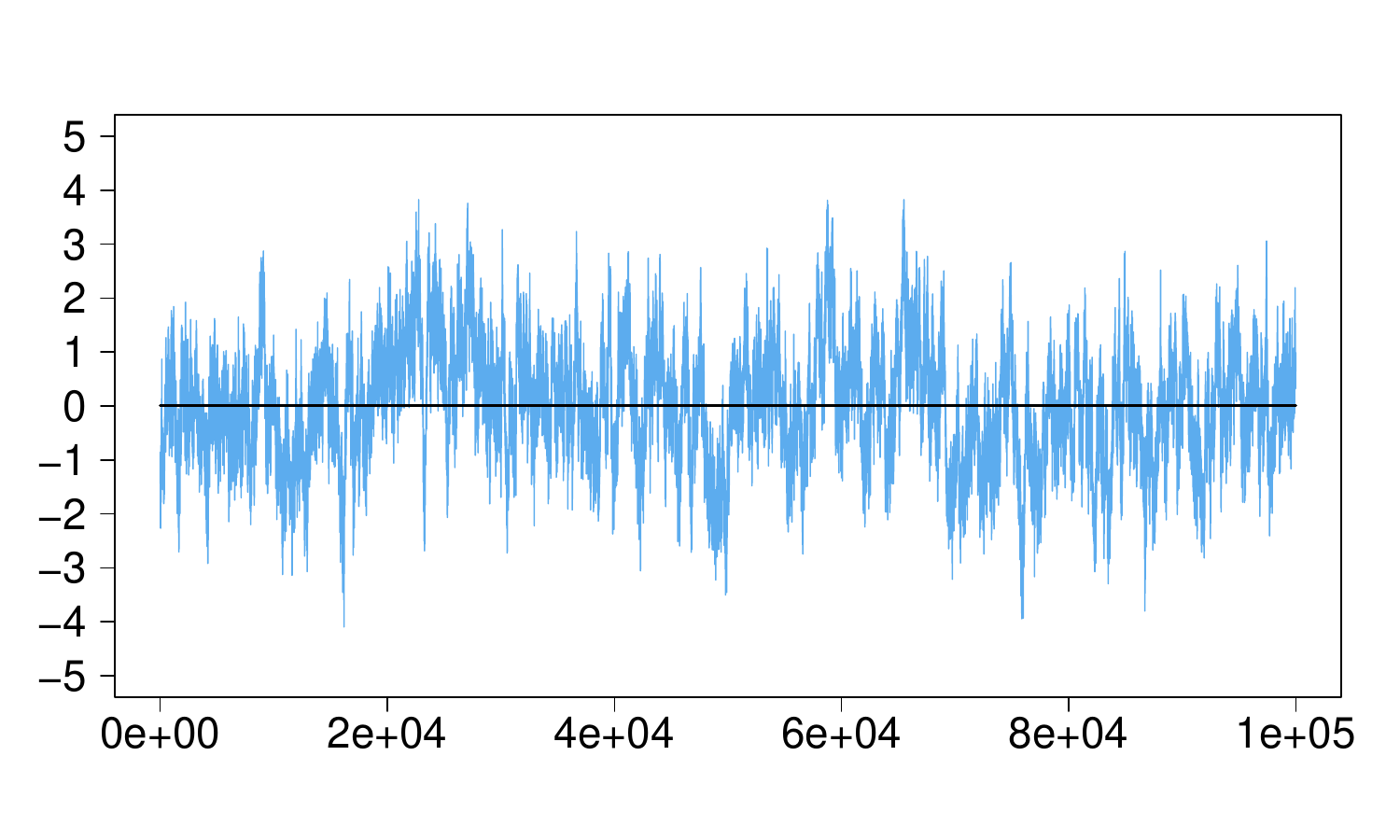}
	} 
	\subfloat[MALA traceplot (two-planet system)]{
		\label{traceplot_figs_b}
		\includegraphics[width=2.3686in]{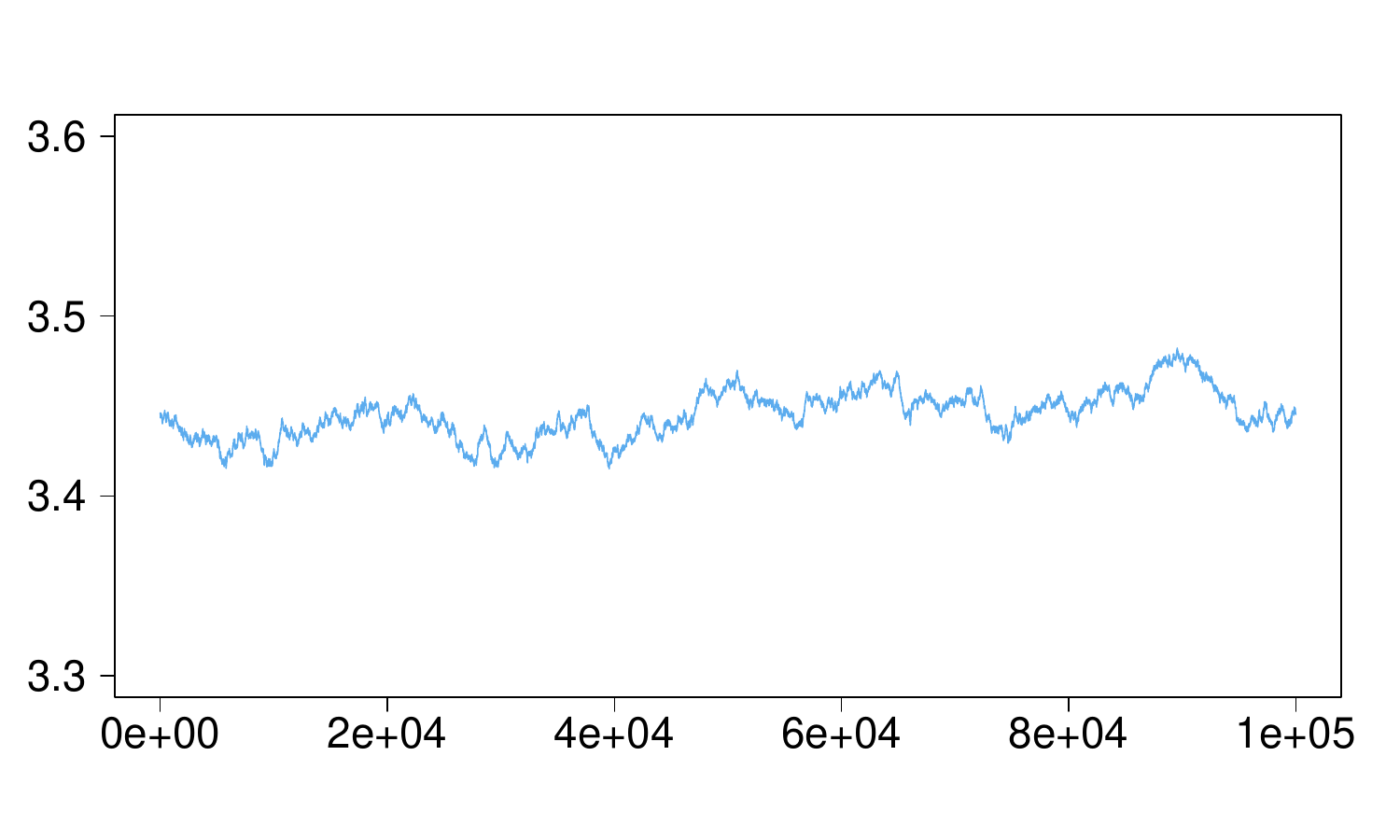}
	} \\
	\subfloat[AM traceplot (t-distribution)]{
		\label{traceplot_figs_c}
		\includegraphics[width=2.3686in]{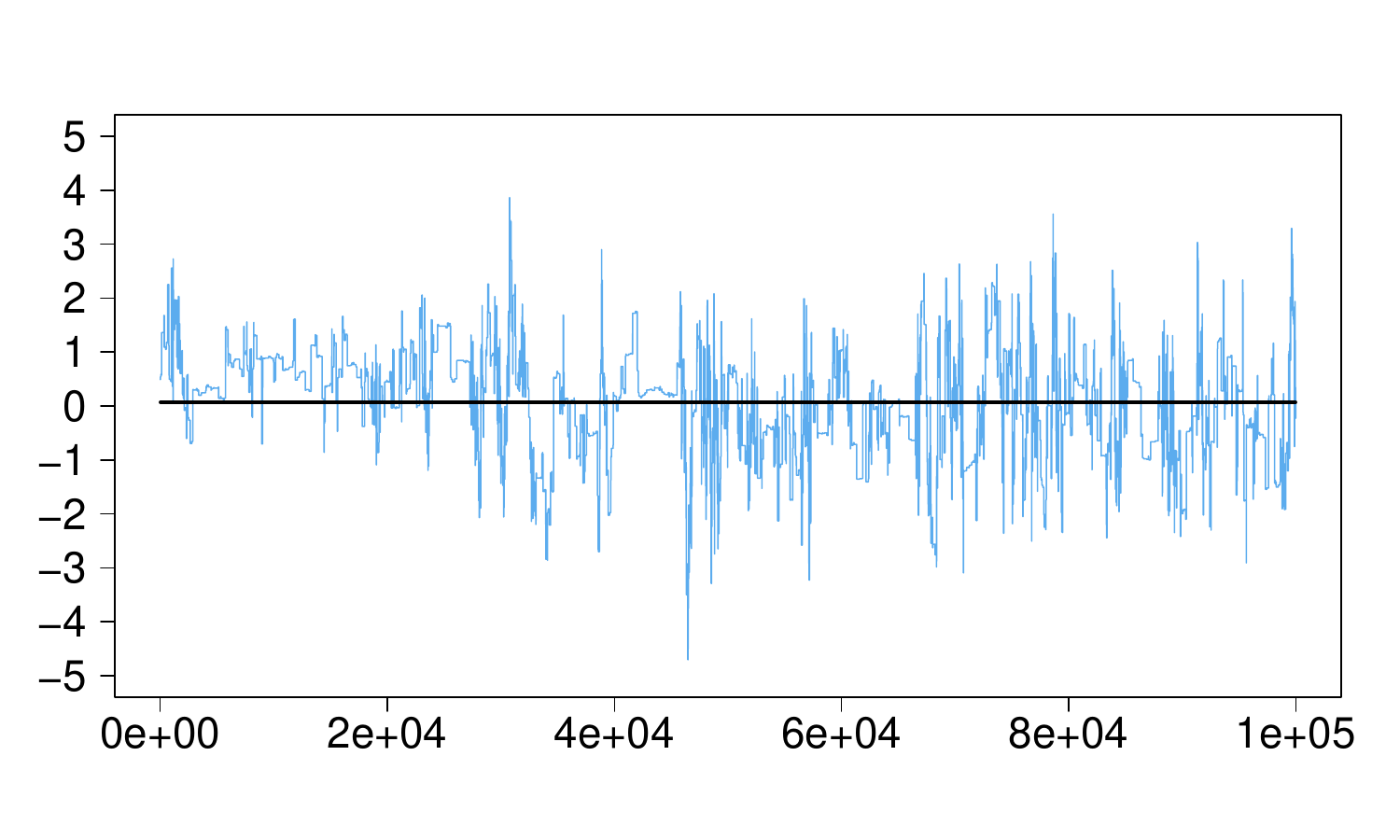}
	} 
	\subfloat[AM traceplot (two-planet system)]{
		\label{traceplot_figs_d}
		\includegraphics[width=2.3686in]{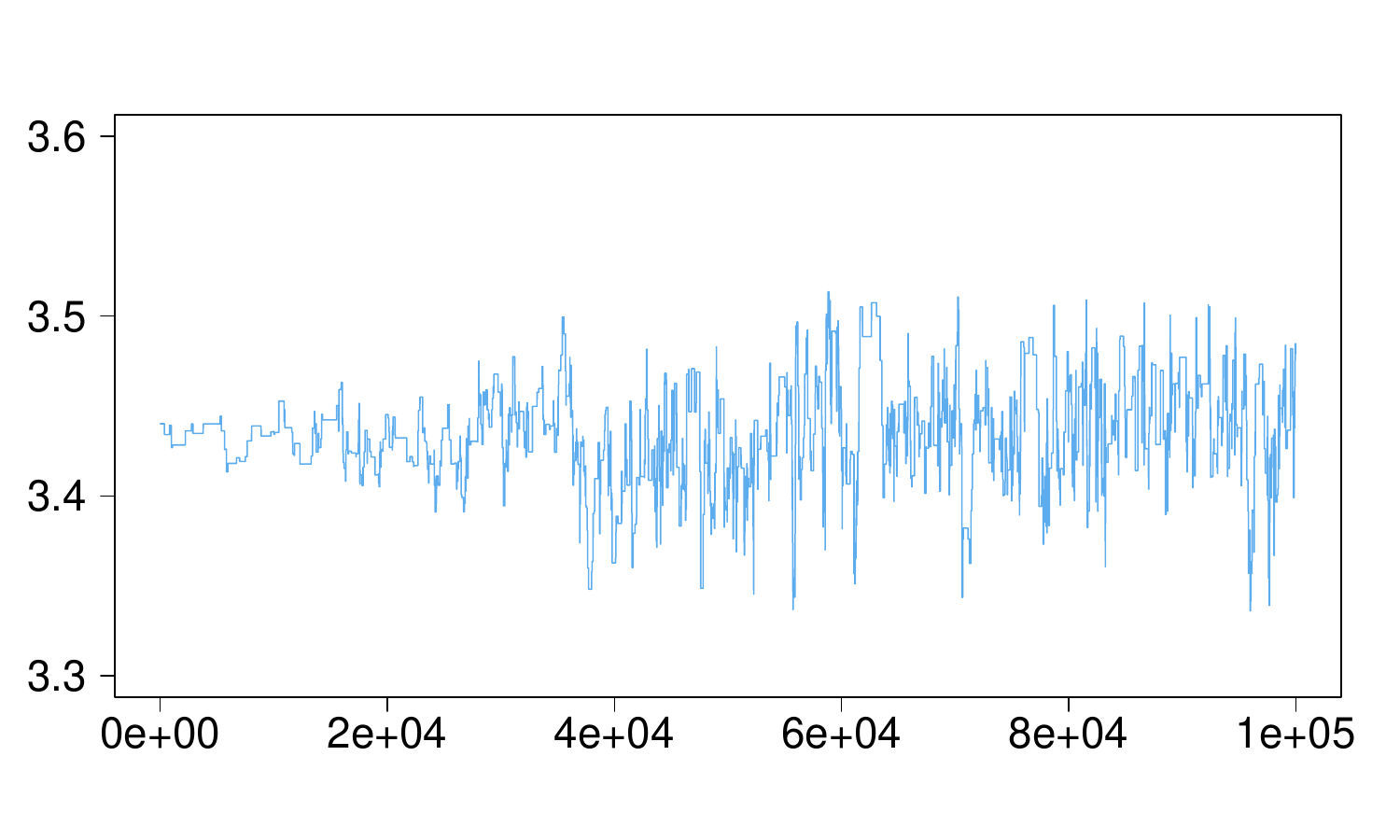}
	} \\
	\subfloat[SMMALA traceplot (t-distribution)]{
		\label{traceplot_figs_e}
		\includegraphics[width=2.3686in]{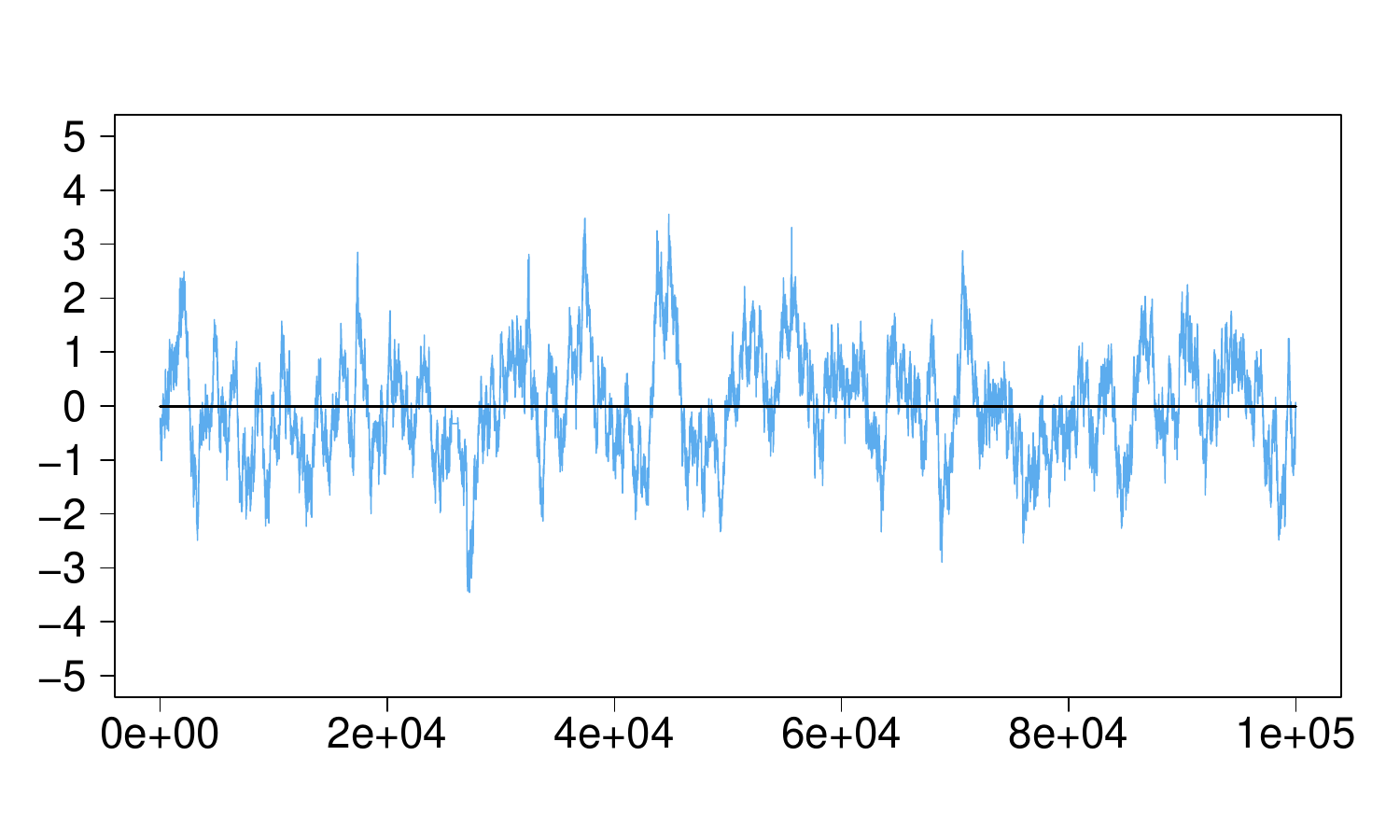}
	} 
	\subfloat[SMMALA traceplot (two-planet system)]{
		\label{traceplot_figs_f}
		\includegraphics[width=2.3686in]{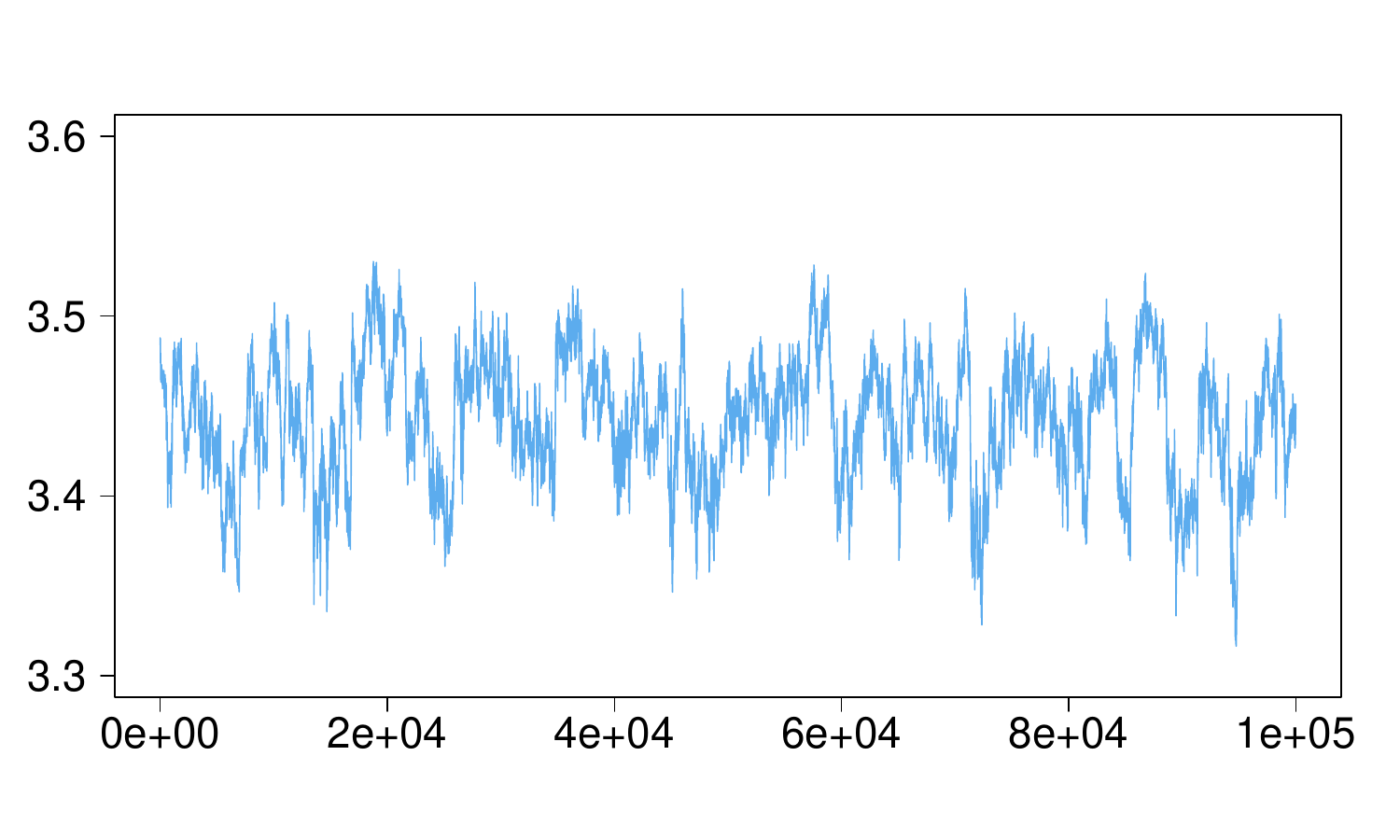}
	} \\
	\subfloat[GAMC traceplot (t-distribution)]{
		\label{traceplot_figs_g}
		\includegraphics[width=2.3686in]{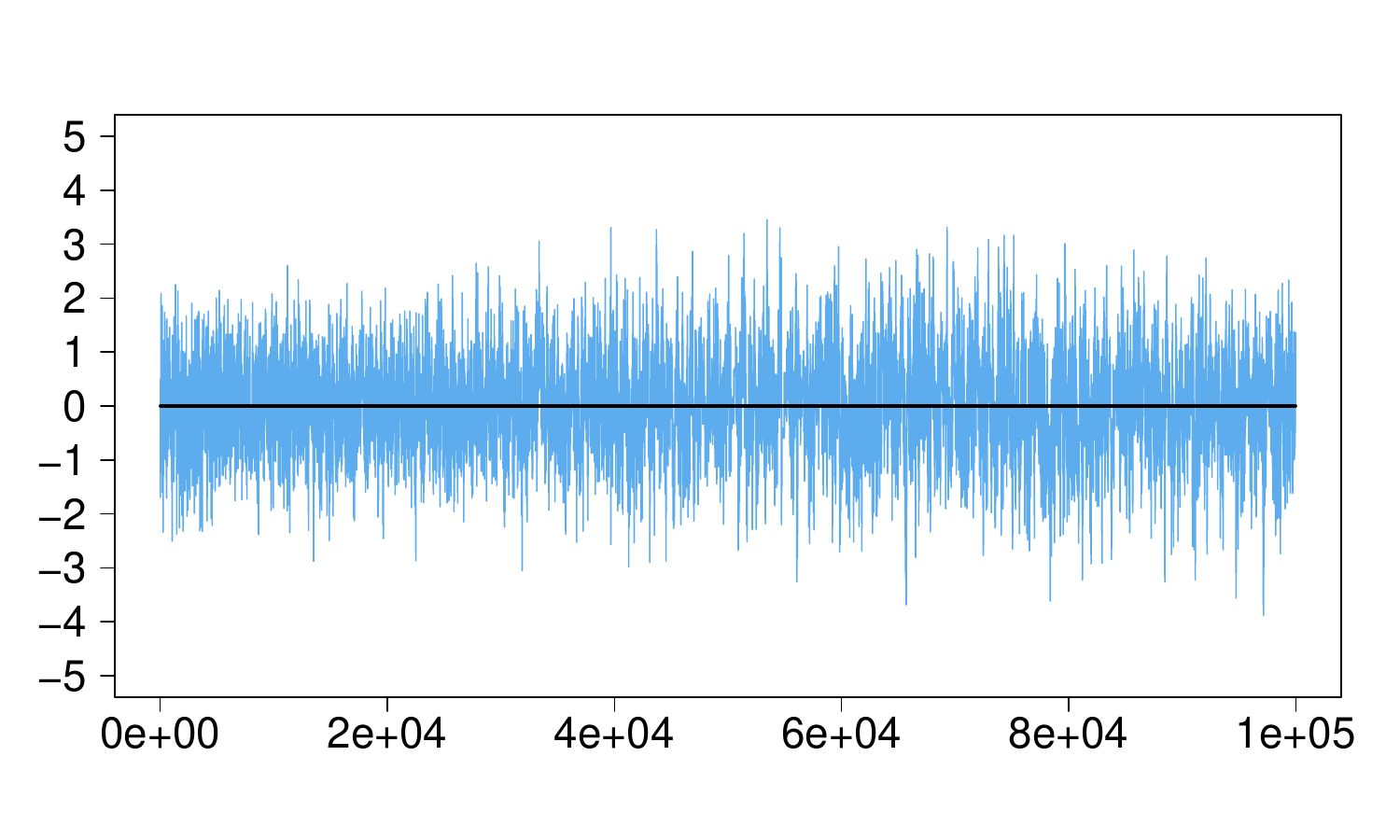}
	} 
	\subfloat[GAMC traceplot (two-planet system)]{
		\label{traceplot_figs_h}
		\includegraphics[width=2.3686in]{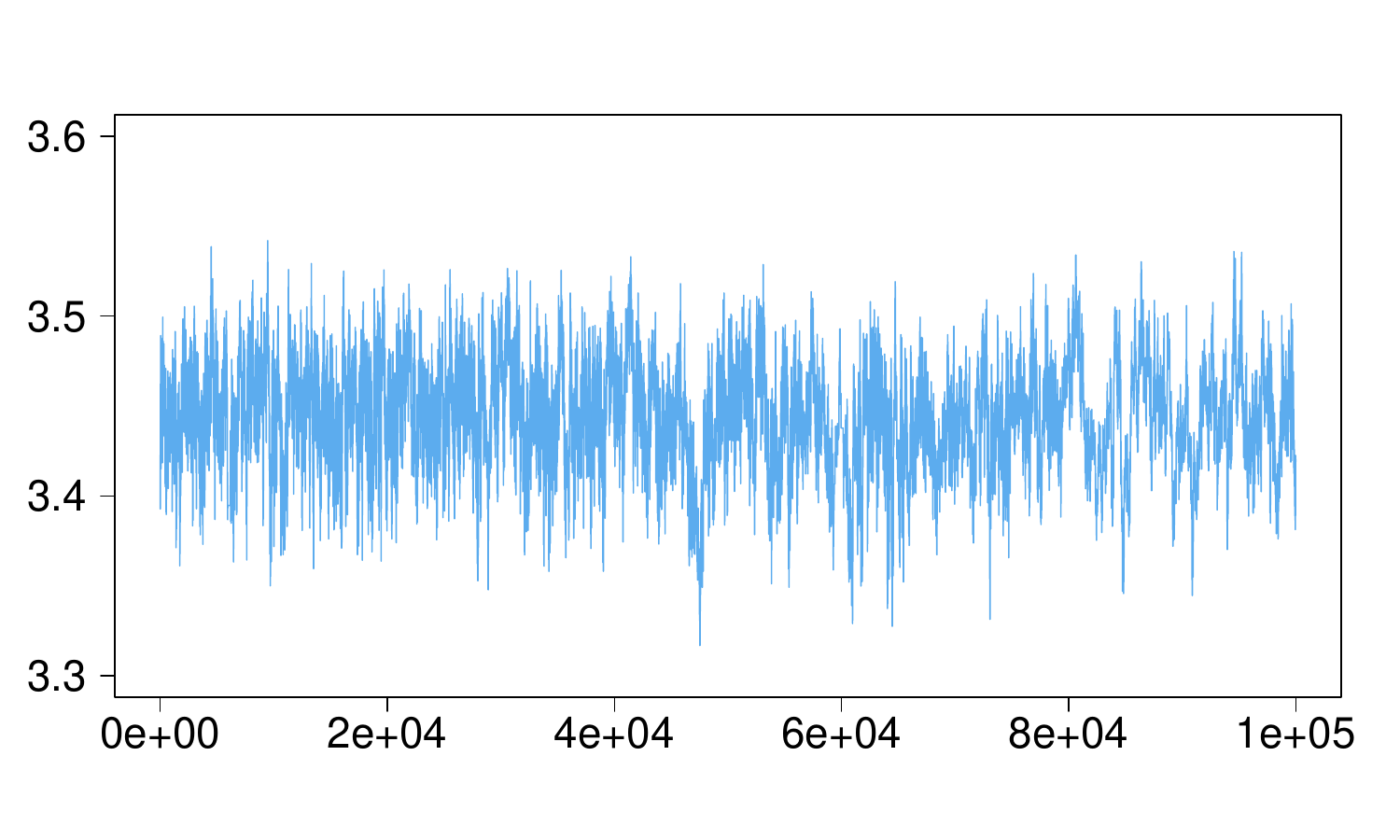}
	} \\
	\caption{Trace plots of single chains
	as a function of Monte Carlo iteration
	corresponding to one of the twenty and 
		eleven parameters of the respective 
		t-distribution and two-planet system. The same chains were used for 
		generating the trace plots of figure 
		\ref{traceplot_figs} and the associated running means and autocorrelations 
		of figure \ref{acf_and_mean_figs}. The 
		black horizontal lines in the t-distribution trace plots represent the true 
		mode.}
	\label{traceplot_figs}
\end{figure}

\subsection{Multivariate t-distribution}

Monte Carlo samples are drawn from an $n$-dimensional Student-$t$ target
$t_{\nu}(0,\frac{\nu-2}{\nu}\Sigma(\xi))$ with $\nu$ degrees of freedom and 
covariance matrix
\begin{equation}
\label{tp:eq:normal:sigma}
\Sigma(\xi)=\left(\begin{array}{ccccc}
1 & \xi^1 & \dots & \xi^{n-2} & \xi^{n-1} \\
\xi^1 & 1 & \dots & \xi^{n-3} & \xi^{n-2} \\
\vdots & \vdots & \ddots & \vdots & \vdots \\
\xi^{n-2} & \xi^{n-3} & \dots & 1 & \xi^{1} \\
\xi^{n-1} & \xi^{n-2} & \dots & \xi^1 & 1 \\
\end{array}\right)
\end{equation}
for some constant $0<\xi<1$ that determines the level of correlation between 
parameter coordinates. The elements of the 
$i$-th diagonal of the $n\cdot n$ covariance matrix $\Sigma(\xi)$ equal 
$\xi^{i-1},~i=1,2,\dots,n$.
The scale matrix $\frac{\nu-2}{\nu}\Sigma(\xi)$ of the $t$-distribution scales 
$\Sigma(\xi)$ by a factor of 
$\frac{\nu-2}{\nu}$ so that the covariance matrix of the $t$-distribution is 
$\Sigma(\xi)$.

In this example, the setting is not Bayesian, so there is no prior distribution 
involved \cite{pap_mir_gir__mon}. Instead, 
MCMC sampling acts as a random number generator to simulate from a 
$t$-distribution 
$t_{\nu}(0,\frac{\nu-2}{\nu}\Sigma(\xi))$. The simulated chains are randomly 
initialized away from the zero mode of 
the $t$-distribution, and they are expected to converge to zero. In other 
words, the zero mode of 
$t_{\nu}(0,\frac{\nu-2}{\nu}\Sigma(\xi))$ is seen as the parameter vector to be 
estimated.

The dimension of the $t$-target is set to $n=20$, relatively high correlation 
is induced by selecting $\xi=0.9$ in 
\eqref{tp:eq:normal:sigma}, and some amount 
of probability mass is maintained in the $t$-distribution tails by choosing 
$\nu=30$ degrees of freedom. The present example 
does not reach the realm of a fully-fledged application, especially in terms of 
log-target complexity, yet it gives a first 
indication of some common computational costs appearing in more realistic 
applications, including automatic differentiation 
and SoftAbs metric evaluations.

Figure \ref{acf_and_mean_figs_c} displays the running means of four chains that 
correspond to the seventeenth coordinate 
$\theta_{17}$ of the twenty-dimensional parameter
$\theta\sim t_{\nu}(0,\frac{28}{30}\Sigma(0.9))$, with a single chain generated 
by each of MALA, AM, 
SMMALA and GAMC. The running mean of the chain simulated using 
MALA does not appear to converge rapidly to the true mode of zero. This accords 
with theoretical knowledge. In particular,
\cite{rob_twe__exp} and \cite{liv_gir__inf} have shown that if a target density 
has tails heavier than exponential or 
lighter than Gaussian, then a MALA proposal kernel does not yield a 
geometrically ergodic Markov chain. Furthermore, it can 
be seen that the chain generated by GAMC converges faster than the chains 
produced by AM, MALA and SMMALA.

Table \ref{speedup_tables} reports the minimum, mean, median and maximum ESS of 
the $n=20$ parameter coordinates. 
As seen from table \ref{speedup_tables}, GAMC achieves roughly ten times larger 
ESS in comparison to AM, MALA and SMMALA 
for the t-distribution example. Figures \ref{acf_and_mean_figs_d}, 
\ref{traceplot_figs_a}, \ref{traceplot_figs_c}, 
\ref{traceplot_figs_e} and \ref{traceplot_figs_g} show the autocorrelation and 
trace plots of the four chains with running 
means presented by figure \ref{acf_and_mean_figs_c}. Figure 
\ref{acf_and_mean_figs_d} demonstrates that GAMC has the 
lowest autocorrelation among the four compared samplers. The trace plots 
provide further circumstantial evidence of the 
faster mixing of GAMC for the Student-t target. The mixing properties of GAMC 
in the case of t-distribution come as a 
surprise, since GAMC was designed to reduce the cost paid for faster mixing 
rather than to achieve the fastest possible 
mixing in absolute terms. GAMC has shorter CPU runtime in comparison to SMMALA, 
but longer runtime than MALA and AM.

With a speed-up of $3.18$, GAMC is about three times more efficient than MALA 
and orders of magnitude more efficient than
AM and SMMALA for the Student-t target $t_{\nu}(0,\frac{28}{30}\Sigma(0.9))$ of 
this example.

\subsection{Radial velocity of a star in planetary systems}

The study of exoplanets has emerged as an important area of modern astronomy.  
While astronomers utilize a variety of 
different methods for detecting and characterizing the properties of exoplanets 
and their orbits, each of the prolific 
methods to date shares several characteristics. First, translating astronomical 
observations into planet physical and 
orbital properties requires significant statistical analysis. Second, 
characterizing planetary systems with multiple 
planets requires working with high-dimensional parameter spaces. Third, the 
posterior probability densities are often 
complex with correlated parameters, non-linear correlations or multiple 
posterior modes. MCMC has proven invaluable for 
providing accurate estimates of planet properties and orbital parameters and is 
now used widely in the field.

For analyzing simple data sets such as one planet detected at high 
signal-to-noise ratio, the random walk 
Metropolis-Hastings sampler is 
effective and the choice of MCMC sampling algorithm is unlikely to be important 
\cite{for__qua}. For analyzing more 
complex data sets such as a star with several planets, more care is necessary 
to avoid poor mixing of the Markov chains. One 
approach is ``artisinal'' MCMC, where proposal densities are hand-crafted for a 
particular problem by making use of 
physical intuition and validation on simulated data sets \cite{for__im}. 
However, it is desirable to identify more 
sophisticated algorithms that can be efficient with minimal tuning or human 
intervention. Here, GAMC is applied to 
simulated radial velocity planet search data sets so as to illustrate the 
potential of the sampler for future astronomical 
or other scientific applications.

One prolific method for characterizing the orbits of extrasolar planets is the 
radial velocity method.  
Astronomers make a series of precise measurements of the line-of-sight velocity 
of a target star. The velocity of the star 
$v(t)$ changes with time $t$ due to the gravitational tug of any planets 
orbiting it. A basic radial velocity data set 
consists of a list of $n_d$ observation times $t_i$, $i=0,1,\dots,n_d-1$, and 
measured velocities $\hat{v}_i$.

The observed velocity $\hat{v}_i$ of the star at time $t_i$ is modelled as the 
unknown velocity $v(t_i)$ plus some
measurement error $\epsilon_i$, as seen in \eqref{observed_v}. For many 
planetary systems with $n_p$ planets, the stellar 
line-of-sight velocity $v(t_i)$ can typically be well approximated by 
\eqref{approximated_v}. Independent Gaussian 
measurement errors $\epsilon_i$ with variances $\sigma_i^2$ are assumed 
according to \eqref{star_error}. 
\eqref{observed_v}, \eqref{approximated_v} and \eqref{star_error} introduce the 
following model for the radial velocity of a 
star in a planetary system consisting of $n_p$ planets:
\begin{align}
\hat{v}_i &= v(t_i)+\epsilon_i,
\label{observed_v}\\
v(t_i) &= C\displaystyle\sum_{j=1}^{n_p} K_j ( 
\cos{(\omega_j+T(t_i,P_j,e_j,M_{0,j}))}+ e_j \cos{(\omega_j))},
\label{approximated_v}\\
\epsilon_i &\sim \mathcal{N}(0, \sigma_i^2).
\label{star_error}
\end{align}

In 
\eqref{approximated_v},
$C$ is the systemic line-of-sight velocity of the planetary system,
$K_j$ is the velocity amplitude induced by the $j$-th planet,
$P_j$ is the orbital period of the $j$-th planet, 
$e_j$ is the orbital eccentricity of the $j$-th planet,
$M_{0,j}$ is the mean anomaly at time $t_{0}=0$ of the $j$-th planet,
$\omega_j$ is the argument of pericenter of the $j$-th planet, and
$T(t_i,P_j,e_j)$ is the true anomaly at time $t_i$ of the $j$-th planet.

The true anomaly $T$ is an angle that specifies the location of the planet and 
star along their orbit at a given time. The 
true anomaly $T$ is related to the eccentric anomaly $E$ by $\tan(T/2) = 
\sqrt{\frac{1+e}{1-e}} \tan(E/2)$.
The eccentric anomaly $E$ can be calculated from the mean anomaly $M$ from 
Kepler's equation, $M = E - E \sin {(E)}$, and an 
iterative solver. The mean anomaly $M$ increases at a linear rate with time $t$ 
according to the
equation $M(t) = M_0 + 2\pi t/P$.  

A parameter vector
$(K_j,P_j,e_j,M_{0,j},\omega_j)$
of length five
is associated with the $j$-th planet,
as seen from \eqref{approximated_v}. Thus, a total of 
$n=5n_p+1$ model parameters
\begin{equation*}
\theta=(
C,
(K_1,P_1,e_1,M_{0,1},\omega_1),
\ldots,(K_{n_p},P_{n_p},e_{n_p},M_{0,{n_p}},\omega_{n_p}))
\end{equation*}
appear in a planetary 
system with $n_p$ planets. The notation $v(t_i,\theta)$ can be used in place of 
$v(t_i)$ to indicate that the 
stellar line-of-sight velocity \eqref{approximated_v} of the star depends on 
the parameters $\theta$.

According to \eqref{star_error}, the sum of squares of the normalized 
measurement errors $\epsilon_i/\sigma_i$ follow a 
chi-squared distribution with $n_d$ degrees of freedom,
\begin{equation}
\label{chi_sq}
\sum_{i=0}^{n_d-1}\left(\frac{\epsilon_i}{\sigma_i}\right)^2\sim\chi_{n_d}^2.
\end{equation}
The log-likelihood arises from \eqref{chi_sq} and \eqref{observed_v} as
\begin{equation}
\label{astro_lik}
\mathcal{L}(t,\hat{v},\sigma|\theta) =
-\frac{1}{2}\sum_{i=0}^{n_d-1}
\left(\frac{v(t_i,\theta)-\hat{v}_i}{\sigma_i}\right)^2,
\end{equation}
where
$t=(t_0,t_1,\dots,t_{n_d-1})$,
$\hat{v}=(\hat{v}_0,\hat{v}_1,\dots,\hat{v}_{n_d-1})$,
$\sigma=(\sigma_0,\sigma_1,\dots,\sigma_{n_d-1})$.
It is assumed that the measurement uncertainties $\sigma$ are known, so $t$, 
$\hat{v}$ 
and $\sigma$ make up the available data.

For the relatively simple model described by \eqref{astro_lik} and 
\eqref{approximated_v}, astronomers commonly use a set of 
priors elicited during a 2013 SAMSI program on astrostatistics. Modified 
Jeffreys priors are adopted for the velocity 
amplitudes $K_j$ and orbital periods $P_j$.
Uniform priors are employed for the orbital eccentricities $e_j$, velocity 
offsets $M_{0,j}$ and angle offsets $\omega_j$ according to
$e_j\sim\mathcal{U}[0,1)$, $\omega_j\sim\mathcal{U}[0,2\pi)$, 
$M_{0,j}\sim\mathcal{U}[0,2\pi)$.

GAMC is benchmarked on two simulated data sets, of which one consists of 
$n_p=1$ planet and the other one comprises $n_p=2$
planets. In each case, $n_d=50$ observed velocities $\hat{v}\in\mathbb{R}^{50}$ 
are simulated at time points 
$t\in\mathbb{R}^{50}$ spread uniformly over two years.

\subsubsection{One-planet system}
For the one-planet system, the isochronal velocities $\hat{v}_i$ are simulated 
using $C = 1.0$, $K_1 = 20m/s$, $P_1 = 50$ 
days, $e_1 = 0.2$, $M_{0,1} = \pi/4$, $\omega_1 = \pi/4$ and $\sigma_i = 2m/s$ 
for $i=0,1,\dots,49$. The parameter vector 
for this one-planet system is
\begin{equation*}
\theta=(C,K_1,P_1,e_1,M_{0,1},\omega_1),
\end{equation*}
so $n=6$ parameters are
simulated from the target built upon log-likelihood \eqref{astro_lik}.

Figure \ref{acf_and_mean_figs_e} shows the running means of four chains for the 
velocity amplitude $K_1$ induced by the 
single planet, with one chain generated by each of the four compared samplers. 
SMMALA and GAMC seem to converge to the same
value, although the latter appears to converge faster.

Table \ref{speedup_tables} provides the minimum, mean, median and maximum ESS 
of the $n=6$ astronomical parameters 
$\theta=(C,K_1,P_1,e_1,M_{0,1},\omega_1)$. As seen from table 
\ref{speedup_tables} and figure 
\ref{acf_and_mean_figs_f}, GAMC exhibits the largest ESS and smallest 
autocorrelation and therefore appears to have the 
fastest mixing. In terms of CPU runtime, GAMC is about three times faster than 
SMMALA, but slower than MALA and AM.

Apart from attaining the fastest mixing, GAMC outperforms MALA by a factor of 
$246.59$ in terms of speed-up and SMMALA by 
a factor of even higher order of magnitude. GAMC has the second-best 
efficiency, with AM being the most efficient by 
reaching a $378.50$ speed-up in comparison to MALA. The higher efficiency of AM 
over GAMC for this example is attributed 
to the relatively small dimension $n=6$ of the parameter space. GAMC might 
still be preferred over AM for this 
low-dimensional one-planet system, considering that the former sampler has 
higher ESS and higher acceptance rate than the 
latter at a relatively modest additional computational cost.

\subsubsection{Two-planet system}

For the two-planet system, the isochronal velocities $\hat{v}_i$ are simulated 
using
$C = 1.0$,
$K_1 = 30m/s$, $P_1 = 40$ days, $e_1 = 0.2$, $M_{0,1} = \pi/4$, $\omega_1 = 
\pi/4$,
$K_2 = 30m/s$, $P_2 = 80.8$ days, $e_2 = 0.2$, $M_{0,2} = \pi/4$, $\omega_2 = 
\pi/4$,
and
$\sigma_i = 2m/s$ for $i=0,1,\dots,49$. The parameter vector
\begin{equation*}
\theta=(C,K_1,P_1,e_1,M_{0,1},\omega_1,K_2,P_2,e_2,M_{0,2},\omega_2)
\end{equation*}
is associated with this two-planet system, so $n=11$ parameters are
simulated from the target built upon log-likelihood \eqref{astro_lik}.

Figure \ref{acf_and_mean_figs_g} displays the running means of four chains for 
the velocity amplitude $K_1$ induced by 
planet one of the two-planet system, with one chain generated by each of the 
four compared samplers. GAMC seems to 
converge the fastest, followed by SMMALA. AM does not show signs of 
convergence, which is related to the low acceptance rate 
of AM in this example.

Table \ref{speedup_tables} provides the minimum, mean, median and maximum ESS 
of the
eleven parameters.
Similarly to the one-planet system, 
GAMC attains the highest ESS, lowest autocorrelation and most rapidly mixing 
trace in the case of the two-planet system, 
as seen from table \ref{speedup_tables} and figures \ref{acf_and_mean_figs_h}, 
\ref{traceplot_figs_b}, 
\ref{traceplot_figs_d}, \ref{traceplot_figs_f} and \ref{traceplot_figs_h}.

The MALA trace plot of figure \ref{traceplot_figs_b} is characterized by slow 
exploration of the state space of parameter 
$K_1$, which is attributed to the small stepsize required for maintaining an 
acceptance rate close to the optimal rate of 
$57.4\%$. Although the AM chain of figure \ref{traceplot_figs_d} takes longer 
proposal steps than its MALA counterpart, AM 
has a very low acceptance rate of $0.01\%$. SMMALA offers a substantial 
improvement in mixing over MALA and AM 
according to figure \ref{traceplot_figs_f}, while GAMC appears to have the most 
rapid mixing among the four samplers 
(figure \ref{traceplot_figs_h}).

GAMC ranks third in absolute runtime behind AM and MALA for the system of two 
planets (table \ref{speedup_tables}). 
However, AM does not work in the case of two planets, since it fails to 
converge and it has a prohibitively low acceptance 
rate of $0.01\%$. Besides, MALA does not seem to converge either and it 
explores the state space very slowly. In fact, the 
superiority of GAMC in this example is depicted by the largest ESS and highest 
overall efficiency, with a relative 
speed-up about $25$ and $20$ times higher than MALA and SMMALA, respectively.

\subsection{Synopsis of empirical results from simulations}

GAMC has the highest ESS and thus the fastest mixing in all three examples.
This empirical finding might indicate that some random proposal kernels 
or combinations of proposal kernels have better mixing properties than proposal 
mechanisms based on solitary geometric 
kernels.

In the two most computationally demanding
examples 
(t-distribution and two-planet system), GAMC manifested
its capacity to achieve the highest speed-up among its competing samplers. 
Thus, combining kernels might help achieve high 
mixing per step with low computational cost per step for a range of expensive 
models.

Simulations have led empirically to an optimal acceptance rate between $20\%$ 
and $40\%$ for GAMC. This might
be explained by the fact that AM contributes the majority of Monte Carlo steps 
to GAMC for relatively small values of the
tuning parameter $r$ in \eqref{exp_schedule}.

\section{Discussion}


This paper initiates a conceptually straightforward, yet potentially powerful, 
approach to the problem of making manifold 
MCMC algorithms more computationally accessible. The main idea is to combine 
geometric and non-geometric proposal kernels 
to find a balance between computational cost and fast mixing. GAMC has been 
empirically validated on a t-distribution 
and on 
astronomical models of planetary systems. The
initial simulation studies of this paper,
along with applications of GAMC on exoplanet transit timing variation models in
\cite{tuchow2019},
reveal the 
potential of GAMC in terms of sampling efficiency relative to algorithms such 
as SMMALA. 

MCMC algorithms that exploit geometric information about the posterior shape 
are likely to be more efficient in terms of the 
absolute number of model evaluations. Manifold MCMC methods could make it 
practical to generate posterior samples with 
increased effective sample sizes. Unfortunately, computing partial derivatives 
for every proposal, as required for MMALA or 
SMMALA, would be extremely expensive.

GAMC algorithms have the potential to significantly reduce the number of 
gradient 
and Hessian evaluations, and are thus expected to accelerate computations by 
over an order of magnitude relative to 
SMMALA for expensive models.
Exploring ways of injecting local geometric information in adaptive
or other non-geometric MCMC methods promises to make manifold MCMC more 
amenable to realistic applications.
GAMC opens up possible avenues of methodological research for building proposal 
mechanisms based on random proposal kernels.

\section*{Acknowledgments} 	E.B.F. acknowledges the support of the Eberly 
College of Science, Center for 
Astrostatistics, Institute for CyberScience and Center for Exoplanets and 
Habitable Worlds of Pennsylvania State 
University, USA. The Center for Exoplanets and Habitable Worlds is supported by 
the Pennsylvania State University, the 
Eberly College of Science, and the Pennsylvania Space Grant Consortium. The 
results reported herein benefited from 
collaborations and/or information exchange within the Nexus for Exoplanet 
System Science (NExSS) research coordination 
network sponsored by the Science Mission Directorate (SMD) of NASA.



\begin{thebibliography}{99}
	
	\bibitem{and_mou__ont}
	\newblock C.~Andrieu and E.~Moulines,
	\newblock On the ergodicity properties of some adaptive {MCMC} algorithms,
	\newblock \emph{The Annals of Applied Probability}, \textbf{16} (2006),
	1462--1505.
	
	\bibitem{bai_rob_ros__ont}
	\newblock Y.~Bai, G.~O. Roberts and J.~S. Rosenthal,
	\newblock On the containment condition for adaptive {M}arkov chain {M}onte
	{C}arlo algorithms,
	\newblock \emph{Advances and Applications in Statistics}, \textbf{21}.
	
	\bibitem{bet__age}
	\newblock M.~Betancourt,
	\newblock A general metric for {R}iemannian manifold {H}amiltonian {M}onte
	{C}arlo,
	\newblock in \emph{Proceedings of 1s international conference on {G}eometric
		{S}cience of {I}nformation} (eds. F.~Nielsen and F.~Barbaresco),
	\newblock Springer, Berlin Heidelberg, 2013,
	\newblock 327--334.
	
	\bibitem{cal_eps_sil__bay}
	\newblock B.~Calderhead, M.~Epstein, L.~Sivilotti and M.~Girolami,
	\newblock Bayesian approaches for mechanistic ion channel modeling,
	\newblock in \emph{In silico systems biology} (ed. V.~M. Schneider),
	\newblock Humana Press, Totowa, NJ, 2013,
	\newblock chapter~13, 247--272.
	
	\bibitem{cal_gir__sta}
	\newblock B.~Calderhead and M.~Girolami,
	\newblock Statistical analysis of nonlinear dynamical systems using
	differential geometric sampling methods,
	\newblock \emph{Interface Focus}, \textbf{1}.
	
	\bibitem{chi_gre__und}
	\newblock S.~Chib and E.~Greenberg,
	\newblock Understanding the {M}etropolis-{H}astings algorithm,
	\newblock \emph{The American Statistician}, \textbf{49} (1995), 327--335.
	
	\bibitem{sim_bad_cem__sto}
	\newblock U.~\c{S}im\c{s}ekli, R.~Badeau, A.~T. Cemgil and G.~Richard,
	\newblock Stochastic quasi-{N}ewton {L}angevin {M}onte {C}arlo,
	\newblock in \emph{Proceedings of the 33rd {I}nternational {C}onference on
		{M}achine {L}earning} (eds. M.~F. Balcan and K.~Q. Weinberger),
	\newblock Proceedings of Machine Learning Research, 2016,
	\newblock 642--651.
	
	\bibitem{dav_sto__imp}
	\newblock S.~Davie and A.~J. Stothers,
	\newblock Improved bound for complexity of matrix multiplication,
	\newblock \emph{Proceedings of the Royal Society of Edinburgh, Section: A
		Mathematics}, \textbf{143} (2013), 351--369.
	
	\bibitem{dua_ken_pen__hyb}
	\newblock S.~Duane, A.~Kennedy, B.~J. Pendleton and D.~Roweth,
	\newblock Hybrid {M}onte {C}arlo,
	\newblock \emph{Physics Letters B}, \textbf{195} (1987), 216--222.
	
	\bibitem{for__qua}
	\newblock E.~B. Ford,
	\newblock Quantifying the uncertainty in the orbits of extrasolar planets,
	\newblock \emph{The Astronomical Journal}, \textbf{129} (2005), 1706--1717.
	
	\bibitem{for__im}
	\newblock E.~B. Ford,
	\newblock Improving the efficiency of {M}arkov chain {M}onte {C}arlo for
	analyzing the orbits of extrasolar planets,
	\newblock \emph{The Astrophysical Journal}, \textbf{642} (2006), 505--522.
	
	\bibitem{leg__pow}
	\newblock F.~L. Gall,
	\newblock Powers of tensors and fast matrix multiplication,
	\newblock in \emph{Proceedings of the 39th international symposium on symbolic
		and algebraic computation},
	\newblock Association for Computing Machinery, 2014,
	\newblock 296--303.
	
	\bibitem{gey__pra}
	\newblock C.~J. Geyer,
	\newblock Practical {M}arkov chain {M}onte {C}arlo,
	\newblock \emph{Statistical Science}, \textbf{7} (1992), 473--483.
	
	\bibitem{gill_gol_wal__met}
	\newblock P.~E. Gill, G.~H. Golub, W.~Murray and M.~A. Saunders,
	\newblock Methods for modifying matrix factorizations,
	\newblock \emph{Mathematics of Computation}, \textbf{28} (1974), 505--535.
	
	\bibitem{gir_cal__rie}
	\newblock M.~Girolami and B.~Calderhead,
	\newblock Riemann manifold {L}angevin and {H}amiltonian {M}onte {C}arlo
	methods,
	\newblock \emph{Journal of the Royal Statistical Society: Series B (Statistical
		Methodology)}, \textbf{73} (2011), 123--214.
	
	\bibitem{gri__ona}
	\newblock A.~Griewank,
	\newblock On automatic differentiation and algorithmic linearization,
	\newblock \emph{Pesquisa Operacional}, \textbf{34} (2014), 621--645.
	
	\bibitem{gri_wal__eva}
	\newblock A.~Griewank and A.~Walther,
	\newblock \emph{Evaluating derivatives: principles and techniques of
		algorithmic differentiation},
	\newblock 2nd edition,
	\newblock no. 105 in Other Titles in Applied Mathematics, SIAM, Philadelphia,
	PA, 2008.
	
	\bibitem{gri_wal__ona}
	\newblock J.~E. Griffin and S.~G. Walker,
	\newblock On adaptive {M}etropolis-{H}astings methods,
	\newblock \emph{Statistics and Computing}, \textbf{23} (2013), 123--134.
	
	\bibitem{haa_lai_mir__dra}
	\newblock H.~Haario, M.~Laine, A.~Mira and E.~Saksman,
	\newblock {DRAM}: efficient adaptive {MCMC},
	\newblock \emph{Statistics and Computing}, \textbf{16} (2006), 339--354.
	
	\bibitem{haa_sak_tam__ana}
	\newblock H.~Haario, E.~Saksman and J.~Tamminen,
	\newblock An adaptive {M}etropolis algorithm,
	\newblock \emph{Bernoulli}, \textbf{7} (2001), 223--242.
	
	\bibitem{haj__coo}
	\newblock B.~Hajek,
	\newblock Cooling schedules for optimal annealing,
	\newblock in \emph{Open problems in communication and computation} (eds. T.~M.
	Cover and B.~Gopinath),
	\newblock Springer New York, New York, NY, 1987,
	\newblock 147--150.
	
	\bibitem{hig__com01}
	\newblock N.~J. Higham,
	\newblock Computing a nearest symmetric positive semidefinite matrix,
	\newblock \emph{Linear Algebra and its Applications}, \textbf{103} (1988),
	103--118.
	
	\bibitem{hig__com02}
	\newblock N.~J. Higham,
	\newblock Computing the nearest correlation matrix - a problem from finance,
	\newblock \emph{IMA Journal of Numerical Analysis}, \textbf{22} (2002),
	329--343.
	
	\bibitem{hig_stra__and}
	\newblock N.~J. Higham and N.~Strabi\'{c},
	\newblock Anderson acceleration of the alternating projections method for
	computing the nearest correlation matrix,
	\newblock \emph{Numerical Algorithms}, \textbf{72} (2016), 1021--1042.
	
	\bibitem{hou__hes}
	\newblock T.~House,
	\newblock Hessian corrections to the {M}etropolis adjusted {L}angevin
	algorithm,
	\newblock \emph{arXiv}.
	
	\bibitem{kal__ran}
	\newblock O.~Kallenberg,
	\newblock \emph{Random measures, theory and applications},
	\newblock Springer, 2017.
	
	\bibitem{kir_gel_vec__opt}
	\newblock S.~Kirkpatrick, C.~D. Gelatt and M.~P. Vecchi,
	\newblock Optimization by simulated annealing,
	\newblock \emph{Science}, \textbf{220} (1983), 671--680.
	
	\bibitem{kle__ada}
	\newblock T.~S. Kleppe,
	\newblock Adaptive step size selection for {H}essian-based manifold {L}angevin
	samplers,
	\newblock \emph{Scandinavian Journal of Statistics}, \textbf{43} (2016),
	788--805.
	
	\bibitem{lan_tha_chr__emu}
	\newblock S.~Lan, T.~Bui-Thanh, M.~Christie and M.~Girolami,
	\newblock Emulation of higher-order tensors in manifold {M}onte {C}arlo methods
	for {B}ayesian inverse problems,
	\newblock \emph{Journal of Computational Physics}, \textbf{308} (2016),
	81--101.
	
	\bibitem{liv_gir__inf}
	\newblock S.~Livingstone and M.~Girolami,
	\newblock Information-geometric {M}arkov chain {M}onte {C}arlo methods using
	diffusions,
	\newblock \emph{Entropy}, \textbf{16} (2014), 3074.
	
	\bibitem{loc__sim}
	\newblock M.~Locatelli,
	\newblock Simulated annealing algorithms for continuous global optimization:
	convergence conditions,
	\newblock \emph{Journal of Optimization Theory and Applications}, \textbf{104}
	(2000), 121--133.
	
	\bibitem{mar_sie__aco}
	\newblock J.~F.~D. Martin and J.~M.~R. {n}o Sierra,
	\newblock A comparison of cooling schedules for simulated annealing,
	\newblock \emph{Encyclopedia of Artificial Intelligence}, 344--352.
	
	\bibitem{nea__bay}
	\newblock R.~M. Neal,
	\newblock \emph{Bayesian learning for neural networks}, vol. 118,
	\newblock Springer, 1996.
	
	\bibitem{nev__the}
	\newblock J.~Neveu,
	\newblock \emph{Mathematical foundations of the calculus of probability},
	\newblock Holden-Day, Inc, 1965.
	
	\bibitem{nou_and__aco}
	\newblock Y.~Nourani and B.~Andresen,
	\newblock A comparison of simulated annealing cooling strategies,
	\newblock \emph{Journal of Physics A: Mathematical and General}, \textbf{31}
	(1998), 8373--8385.
	
	\bibitem{pap_mir_gir__mon}
	\newblock T.~Papamarkou, A.~Mira and M.~Girolami,
	\newblock Monte {C}arlo methods and zero variance principle,
	\newblock in \emph{Current trends in {B}ayesian methodology with applications}
	(eds. S.~K. Upadhyay, U.~Singh, D.~K. Dey and A.~Loganathan),
	\newblock Chapman and Hall/CRC, 2015,
	\newblock chapter~22, 457--476.
	
	\bibitem{per__prox}
	\newblock M.~Pereyra,
	\newblock Proximal {M}arkov chain {M}onte {C}arlo algorithms,
	\newblock \emph{Statistics and Computing}, \textbf{26} (2016), 745--760.
	
	\bibitem{rev_lub_pap__for}
	\newblock J.~Revels, M.~Lubin and T.~Papamarkou,
	\newblock Forward-mode automatic differentiation in julia,
	\newblock \emph{arXiv}.
	
	\bibitem{rob_ros__opt}
	\newblock G.~O. Roberts and J.~S. Rosenthal,
	\newblock Optimal scaling of discrete approximations to {L}angevin diffusions,
	\newblock \emph{Journal of the Royal Statistical Society: Series B (Statistical
		Methodology)}, \textbf{60} (1998), 255--268.
	
	\bibitem{rob_ros__cou}
	\newblock G.~O. Roberts and J.~S. Rosenthal,
	\newblock Coupling and ergodicity of adaptive {M}arkov chain {M}onte {C}arlo
	algorithms,
	\newblock \emph{Journal of Applied Probability}, \textbf{44} (2007), 458--475.
	
	\bibitem{rob_ros__exa}
	\newblock G.~O. Roberts and J.~S. Rosenthal,
	\newblock Examples of adaptive {MCMC},
	\newblock \emph{Journal of Computational and Graphical Statistics}, \textbf{18}
	(2009), 349--367.
	
	\bibitem{rob_stra__lan}
	\newblock G.~O. Roberts and O.~Stramer,
	\newblock Langevin diffusions and {M}etropolis-{H}astings algorithms,
	\newblock \emph{Methodology And Computing In Applied Probability}, \textbf{4}
	(2002), 337--357.
	
	\bibitem{rob_twe__exp}
	\newblock G.~O. Roberts and R.~L. Tweedie,
	\newblock Exponential convergence of {L}angevin distributions and their
	discrete approximations,
	\newblock \emph{Bernoulli}, \textbf{2} (1996), 341--363.
	
	\bibitem{sak_vih__erg}
	\newblock E.~Saksman and M.~Vihola,
	\newblock On the ergodicity of the adaptive {M}etropolis algorithm on unbounded
	domains,
	\newblock \emph{The Annals of Applied Probability}, \textbf{20} (2010),
	2178--2203.
	
	\bibitem{sch_pap__ews}
	\newblock R.~Schwentner, T.~Papamarkou, M.~O. Kauer, V.~Stathopoulos, F.~Yang,
	S.~Bilke, P.~S. Meltzer, M.~Girolami and H.~Kovar,
	\newblock {EWS}-{FLI1} employs an {E2F} switch to drive target gene expression,
	\newblock \emph{Nucleic Acids Research}.
	
	\bibitem{see__low}
	\newblock M.~Seeger,
	\newblock \emph{Low rank updates for the {C}holesky decomposition},
	\newblock Technical report, University of California, Berkeley, 2004.
	
	\bibitem{tuchow2019}
	\newblock N.~W. Tuchow, E.~B. Ford, T.~Papamarkou and A.~Lindo,
	\newblock The efficiency of geometric samplers for exoplanet transit timing
	variation models,
	\newblock \emph{Monthly Notices of the Royal Astronomical Society},
	\textbf{484} (2019), 3772--3784.
	
	\bibitem{vih__rob}
	\newblock M.~Vihola,
	\newblock Robust adaptive {M}etropolis algorithm with coerced acceptance rate,
	\newblock \emph{Statistics and Computing}, \textbf{22} (2012), 997--1008.
	
	\bibitem{wil__mod}
	\newblock J.~H. Wilkinson,
	\newblock Modern error analysis,
	\newblock \emph{SIAM Review}, \textbf{13} (1971), 548--568.
	
	\bibitem{wil__brea}
	\newblock V.~V. Williams,
	\newblock Breaking the {C}oppersmith-{W}inograd barrier, 2011.
	
	\bibitem{xif_she_liv__lan}
	\newblock T.~Xifara, C.~Sherlock, S.~Livingstone, S.~Byrne and M.~Girolami,
	\newblock Langevin diffusions and the {M}etropolis-adjusted {L}angevin
	algorithm,
	\newblock \emph{Statistics and Probability Letters}, \textbf{91} (2014),
	14--19.
	
\end{thebibliography}

\providecommand{\href}[2]{#2}
\providecommand{\arxiv}[1]{\href{http://arxiv.org/abs/#1}{arXiv:#1}}
\providecommand{\url}[1]{\texttt{#1}}
\providecommand{\urlprefix}{URL }






\end{document}